\documentclass[twoside,leqno,twocolumn]{article}
\usepackage{soda}

\usepackage[numbers]{natbib}
\usepackage[dvipsnames,usenames]{xcolor}
\usepackage{prettyref}
\usepackage[colorlinks=true,pdfpagemode=Usenone,urlcolor=Blue,linkcolor=RoyalBlue,citecolor=OliveGreen,pdfstartview=FitH]{hyperref}
\usepackage[letterpaper, width=6.8in, height=9.3in, marginratio={1:1, 1:1}]{geometry}
\newtheorem{definition}{Definition}
\newtheorem{remark}{Remark}
\newtheorem{assumption}{Assumption}
\usepackage{algorithmic}
\usepackage{amssymb}
\usepackage{amsmath}
\usepackage{breakurl}

\newcommand{\param}{\alpha}

\newcommand{\Ex}{\mathbb{E}}
\newcommand{\OPT}{\text{OPT}}

\newcommand{\rad}{{\rm rad}}
\newcommand{\crad}{{\gamma}}

\newcommand{\comment}[1]{}

\newcommand{\areg}{{\text{avg-regret}}}
\newcommand{\LP}{\text{LP}}
\newcommand{\CP}{\text{CP}}
\newcommand{\IID}{\text{IID}}
\newcommand{\RP}{\text{RP}}

\newcommand{\bs}[1]{\boldsymbol{#1}}

\newcommand{\OCO}{OCO}
\newcommand{\OMD}{OMD}
\newcommand{\OPTest}{\opthat}

\newcommand{\vopt}{\cv^*_t}
\newcommand{\ropt}{r^*_t}
\newcommand{\vplay}{\cv^\dagger_t}
\newcommand{\rplay}{r^\dagger_t}
\newcommand{\rplayavg}{r^\dagger_{\text{avg}}}
\newcommand{\vplayavg}{\cv^\dagger_\text{avg}}
\newcommand{\voptavg}{\cv^*_\text{avg}}
\newcommand{\vavg}{\cv^\dagger_\text{avg}}

\newcommand{\cl}{\text{cl}}		
\newcommand{\real}{\mathbb{R}}
\newcommand{\dist}{\mathcal{D}}
\newcommand{\cu}{\boldsymbol{u}}

\newcommand{\ones}{\mathbf{1}} 
\newcommand{\oneNorm}{\|\mathbf{1}_d\|}

\newcommand{\cv}{\boldsymbol{v}}
\newcommand{\cvS}{{v}}

\newcommand{\thetaV}{\boldsymbol{\theta}}
\newcommand{\phiV}{\boldsymbol{\phi}}

\newcommand{\x}{\bs{x}}
\newcommand{\y}{\bs{y}}	
\newcommand{\z}{\bs{z}}
\newcommand{\g}{\bs{g}}

\newcommand{\addedNikhil}[1]{#1}
\newcommand{\addedShipra}[1]{#1}
\newcommand{\revision}[1]{#1}
\newcommand{\strongc}[1]{#1}

\newcommand{\removedShipra}[1]{}
\newcommand{\commentShipra}[1]{}
\newcommand{\commentNikhil}[1]{}
\newcommand{\todo}[1]{}

\newcommand{\condition}{\left.\right|}
\newcommand{\regOCO}{{\cal R}}

\newcommand{\EQ}[3]{\begin{center} \vspace{#1}$#3$ \vspace{#2}\end{center}}	

\begin{document}
\title{Fast Algorithms for Online Stochastic Convex Programming}
\author{Shipra Agrawal\thanks{Microsoft Research}  \and 
	Nikhil R. Devanur\thanks{Microsoft Research} 
}
\date{}
\maketitle

\begin{abstract}
We introduce the \emph{online stochastic Convex Programming (CP)} problem, a very general version of stochastic online problems which allows arbitrary concave objectives  and convex feasibility constraints. Many well-studied problems like online stochastic packing and covering, online stochastic matching with concave returns, etc. form a special case of online stochastic CP. We present fast algorithms for these problems, which achieve near-optimal regret guarantees {\em for both the i.i.d. and the random permutation models} of stochastic inputs. When applied to the special case online packing,
 our ideas yield a simpler and faster primal-dual algorithm for this well studied problem, which achieves the optimal {\em competitive ratio}.
Our techniques make explicit the connection of primal-dual paradigm and online learning to online stochastic CP.
\end{abstract}

\section{Introduction}

The theory of online matching and its generalizations has been a great success story that has had a significant impact on practice. The problems considered in this area are largely motivated by online advertising, 
and the theory has influenced how real advertising systems are run. 
As an example, the algorithms given by \citet{Devanur2011} are being used at Microsoft, by the ``delivery engine'' that decides which display ads are shown on its ``properties'' such as webpages, Skype, Xbox, etc. 

In one of the most basic problem formulations \addedShipra{in online advertising}, an ``impression" can be allocated to one of many given advertisers, 
assigning an impression $i$ to advertiser $a$ generates a value $v_{ai}$, 
and an advertiser $a$ can be allocated at most $G_a$ impressions.  
The goal is to maximize the value of the allocation. 
In another variant, advertisers pay per click and have budget constraints on their total payment, 
instead of the capacity constraints as above. 
More sophisticated formulations consider the option to show multiple ads on one webpage, 
which means you can pick among various configurations of ads. 
Each configuration still provides some value which is to be maximized, and advertisers have either capacity or 
budget constraints.

While the algorithm in \citet{Devanur2011} (DJSW algorithm) is used in practice,  
the actual problem has some aspects that are not captured by the \removedShipra{formulation in \citet{Devanur2011}} \addedShipra{formulations considered there}. 
For instance, the actual objective function is not just a linear function, 
such as the sum of the values. There is a penalty for ``under-delivering" 
impressions to an advertiser that increases with the amount of under-delivery. 
This translates into an objective that is a concave function of the total number 
of impressions assigned to an advertiser. 
Another consideration is the diversity of the impressions assigned. 
An advertiser targeting a certain segment of the population expects 
a representative sample of the entire population \citep{GhoshMPV09}. 
In order to avoid deviating from this ideal too much, there are certain 
(convex) penalty functions in the objective that punish  such deviations. 
\removedShipra{Currently heuristic adaptations of the basic algorithm are used to handle these extensions, 
without a principled approach. (Shipra: we don't know what others are doing to handle these extensions)}
\addedShipra{The `essentially linear' formulations of online matching or online packing/covering considered in the literature cannot handle these extensions.}
In this paper, we \addedShipra{consider a very general online convex programming framework that can 
	incorporate these extensions,}
\removedShipra{generalize the previous formulations to a very general convex programming framework, }
and present optimal algorithms for it.

An important practical consideration in the design of online algorithms is that 
the time taken by the algorithm in a single step \removedShipra{such as allocating a single impression} should be very small. 
For instance, the decision to allocate an impression must be made in ``real-time'', 
in a matter of milliseconds. 
The DJSW algorithm satisfies this requirement, but requires solving an LP ever so often, 
to estimate the value of an optimum solution.  
In this paper, we give an algorithm that only requires solving a single LP (for {\em online packing} problems), 
making it even faster than the DJSW algorithm.  
This improvement comes from the fact that in our algorithm the error in the estimate of the optimal solution 
only occurs in the second order error bounds and hence we can tolerate much bigger errors in such an estimate.  

From a theoretical point of view, two closely related online stochastic input models have been studied, 
the random permutation and the i.i.d. model. In the random permutation model, an adversary 
picks the {\em set} of inputs, which are then presented to the algorithm in a random order. 
In the i.i.d. model, the adversary picks a distribution over inputs that is unknown to the algorithm, 
and the algorithm receives i.i.d. samples from this distribution. \addedShipra{The random permutation model is stronger than the i.i.d. model, any algorithm that works for the random permutation model also works for the i.i.d model.} 
The difference between these two models is like the difference between sampling with and without replacement. 
This intuition says that the two models should be very similar to each other, 
but the DJSW algorithm was only known to work for the i.i.d model, 
 not for the random permutation model. 
Earlier algorithms by \citet{DH09,AWY2009,Feldman10} worked for the random permutation model 
but gave worse guarantees. 
Recently \citet{KTRV14} gave an algorithm that matched the optimal guarantee of \citet{Devanur2011} 
for the random permutation model, but their algorithm has to solve 
an LP in every step, making it not practical.  
To summarize, the DJSW algorithm is fast and works for the i.i.d.  model but not for the random permutation model. 
The algorithm by \citet{KTRV14} works for the random permutation model but  is slow. 
We get the best of both worlds, 
our algorithm is fast, and works for the random permutation model.
Moreover, our proof formalizes the intuition mentioned earlier that the difference between i.i.d and the random permutation models is like the difference between sampling with and without replacement.

In terms of techniques, the earlier algorithms used dual variables to guide the allocation, whereas the algorithm of \citet{KTRV14} uses a primal only approach, and their result seemed to suggest that primal-only algorithms were more powerful than primal-dual algorithms. 
Our algorithms are primal-dual, and our results show that primal-dual algorithms are equally powerful. 
In fact, even though the DJSW algorithm could be interpreted as a primal-dual algorithm, duality was never used in the analysis.  
Our algorithm is a true primal-dual algorithm in the sense that we explicitly make use of the duality. 
Also, starting from \citet{MSVV}, it was suspected that there is some relation between these problems and online learning or the ``experts" problem, but no formal connection was known. 
We show such a formal connection, all of our algorithms actually use blackbox access to algorithms 
for solving online learning problems. We show how getting better guarantees for these problems 
boils down to getting better ``low-regret'' guarantees for certain online learning problems. 
This also gives much simpler proofs than earlier papers.

To summarize, our contributions are as follows. 
\begin{enumerate}
	\item We present algorithms with optimal guarantees for a very general online convex programming problem, in a stochastic setting.
	\item  Our algorithms are primal-dual algorithms that are fast and simple, and work for the random permutation model. Our proof techniques formalize the intuition that the random permutation and the i.i.d models are not that different. 
	\item We establish a formal connection between these problems and online learning. 
\end{enumerate}

\subsection{Other Related Work}
The seminal paper of \citet{MSVV} introduced the so called ``Adwords" problem, motivated by the allocation of ad slots on search engines, and started a slew of research into generalizations of the online bipartite  matching problem \citep{KVV}. For the worst-case model, the optimal competitive ratio is $1-1/e$, which can be achieved for 
a fairly general setting \citep{BJN, AGKM, FKMMP09, DHKMQ13}. 
A special case of an objective with a concave function was considered in \citet{DJ12}. 

In order to circumvent the impossibility results in the traditional worst-case models, stochastic models such as the random permutation model and the i.i.d model were introduced \citep{GoelMehta, DH09, VVS10, Devanur2011}. The dominant theme for these stochastic models has been asymptotic guarantees, that show that the competitive ratio tends to $1$ as the ``bid-to-budget" ratio tends to $0$ (as was first shown by \citet{DH09}). The focus then is the {\em convergence rate}, the rate at which the competitive ratio tends to 1 as a function of the bid-to-budget ratio. 
\citet{Feldman10,AWY2009} gave improved convergence rates for the random permutation model
and generalized the result to an {\em online packing} problem. 
\revision{Recently, \citet{ChenWang2013} extended these ideas to the concave returns problem of  \citet{DJ12}.  }
\citet{Devanur2011} gave the optimal convergence rate for the online packing problem in the closely related i.i.d. model.  
\citet{KTRV14} matched these bounds for the random permutation model, and further improved the bounds either when the bid-to-budget ratio is large, or when the instances are sparse. 
This line of research has also had significant impact on the practice of ad allocation with most of the big ad allocation platforms using algorithms influenced by these papers \citep{FHKMS10, KMS13,KDDRTB,VeeEC12a,VeeEC12b}.

\addedShipra{Some versions of these problems also appear in literature under the name of `secretary problems'. 
However the dominant theme in research on secretary problems is to aim for a constant competitive ratio while not making any assumption about ``bid-to-budget" ratio (a notable exception is \citep{kleinberg05}). }

Another interesting line of research has been for the case of bipartite matching. 
\citet{FMMM,BK10,SS10} gave algorithms with competitive ratios better than $1-1/e$  for the known distribution case, 
and \citet{KMT11,MY11} did the same for the random permutation model. 
Other variations such as models for combining algorithms from worst-case and average case, and achieving simultaneous guarantees have also been studied \citep{MNS12, MGZ12}.

A closely related problem is called the ``Bandits with Knapsacks" problem \citep{BwK}, which is similar to 
the online stochastic packing problem. The bandit aspect is different: the algorithm picks an ``arm" of the bandit at each time,  and makes observations (cost, reward, etc.), which are i.i.d samples that depend on the arm. 
There is persistence in the available set of choices across time as the arms are persistent. 
In the online packing problem, 
the set of options in one time step are unrelated to the other time steps. 
Due to this, the main aspect of the bandit problem, the explore-exploit trade off in estimating 
the expectations of the observations for all arms, is absent from the online packing problem.  

In an earlier paper \citep{AD14}, we generalized  Bandits with Knapsacks to include general convex constraints 
and concave rewards, which is analogous to our generalization of the online packing to online convex programming here. 
\addedShipra{
Our high level ideas of using Fenchel duality for `linearization' and 
online learning algorithms for estimating the dual variables is inspired by the use of similar ideas in \citep{AD14}. 
Consequently, we obtain algorithms that are very similar looking to those in \citep{AD14}. 
There are some significant differences in the proof techniques, however, due to the differences in the two problems mentioned in the previous paragraph.
Also, the analysis for the random permutation model, and our adaptations (for the online packing problem) to get competitive ratios instead of regret bounds, 
were entirely absent from \citep{AD14}. }

\removedShipra{Our high level idea of using online learning algorithms is inspired by similar use of online learning in \citep{AD14}. As a result, we obtain algorithms that are very similar looking to those in \citep{AD14}. There are significant differences, for instance the analysis for the random permutation model is entirely absent from \citep{AD14}. The versions where there are only constraints are the closest to each other, and our algorithm/analysis is almost identical to the bandit version. The version where there are rewards as well as constraints has new ideas: crucially 
we identify one  parameter that we need to estimate, the rate at which one can increase the objective at the cost of violating a constraint. 
In the bandit version, we anyway estimate ``everything'' so this issue does not arise. 
Our adaptations to the online packing problem to get competitive ratios instead of regret bounds 
are also absent from the earlier work. }

The online packing problem is also closely related to the Blackwell approachability problem \citep{blackwell1956}. 
The use of online learning algorithms to solve the Blackwell approachability problem \citep{blackwell2011} 
is similar to our use of online learning algorithms. 

Concurrently and independently, \citet{GuptaM14}  found results for online linear programming that are similar to some of ours: they also show how to get competitive ratio bounds for the online packing problem in the random permutation model via a connection to the experts problem. 
For the guarantees  that hold ``in expectation", their bounds are the same as ours. 
For the guarantees that hold ``with high probability", they show bounds without an extra $\sqrt{\log T}$ factor that we get. 
They do not consider the more general convex programming framework.

\subsection{Organization:} The Preliminaries section (\prettyref{sec:prelims}) contains the problem and the input model definitions, the statement of the main result and some background material on online learning and Fenchel duality. 
\prettyref{sec:onlyS} illustrates the basic ideas using a special case with only convex feasibility constraints. 
\prettyref{sec:cp} gives the algorithm, results and proof techniques for the general online stochastc convex programming.
\prettyref{sec:packing} gives tighter  bounds for the special case of the online packing problem.

\section{Problem definition and main results}\label{sec:prelims}
The following problem captures a very general setting of online optimization problems with global constraints and utility functions.
\begin{definition}{\it \scshape [Online Stochastic Convex Programming]}   We receive an initial input of a concave function $f$ over a bounded domain $\subseteq \real^d,$ which we may assume is $[0,1]^d$ w.l.o.g,  and a convex set $S\subseteq[0,1]^d$. Subsequently we proceed in steps, at every time step $t= 1, \ldots,T$, we receive a set \addedShipra{$A_t \subseteq [0,1]^d$} of $d$-dimensional vectors.  We have to pick one vector $\vplay\in A_t$ before proceeding to time step $t+1$, using only information until time $t$. Let $\vplayavg :=  \tfrac{1}{T} \sum_{t=1}^T \vplay$. The goal is to
\[ \text{maximize } f(\vplayavg) \text{ subject to } \vplayavg \in S.\] 
  We assume that the instance is {\em always feasible}, i.e.,  there is a choice of $\cv_t \in A_t~ \forall~t$ such that $\tfrac{1}{T} \sum_{t=1}^T \cv_t \in S$. 
\end{definition} 

\subsection{Stochastic Input Models:} 

In the random permutation (\RP) model, there are $T$ sets $X_1, ..., X_T$ fixed in advance but unknown to the algorithm, and these come in a uniformly random order (given by a random permutation $\pi$) 
as the sequence $A_1=X_{\pi(1)}, ..., A_T=X_{\pi(T)}$. 
The number of time steps $T$ is given to the algorithm in advance. 
In the i.i.d, unknown distribution (\IID) model, 
there is a distribution $\mathcal{D}$ over subsets of $[0,1]^d$, and for each $t$, $A_t$ is an independent sample from $\mathcal{D}$. The distribution $\mathcal{D}$ is unknown to the algorithm.

\addedShipra{It is known that the \RP~model is stronger than the \IID~model. The \IID~model can be thought of as a distribution over \RP~instances and therefore any guarantee for  the \RP~model also carries over to the \IID~model. Henceforth, we will consider the \RP~model by default, unless otherwise mentioned. }

\subsection{Benchmarks.} We measure the performance of an algorithm with respect to a benchmark. 
The bechmark for the \RP\ model is the {\em optimal offline} solution, 
i.e. the choice $\vopt \in A_t$ that maximizes the function $f$ of the average of these vectors while making sure that the average lies in $S$. We denote the value of this solution as the benchmark, $\OPT$. This is a deterministic value since it does not depend on the randomness in the input, which is in the order of arrival. 
For the \IID\ model, the offline optimal actually depends on the randomness in the input, 
and  $\OPT$ denotes the expected value of the offline optimal solution. 

\removedShipra{
As the value $\OPT$ as defined for the \IID~model is not so easy to work with, we define a relaxation which gives us an upper bound on $\OPT$. 
We define the following {\em expected instance} for a given distribution $\dist$, 
which is an offline instance of the corresponding convex program, and the optimum 
value of  this instance, $\OPT_\dist$ is an upper bound on $\OPT$. 
The expected instance is as follows: for each set $A$ that is in the support of the distribution $\dist$, pick a vector $\cv_A\in A$ (perhaps probabilistically) 
that maximizes $f(\cu)$ subject to the constraint that $\cu \in  S$, where 
\[ \cu = \Ex_{A\sim \dist} [ \cv_A] .\] 
The expectation is over both the randomness in the input and the randomness used by the algorithm, if any. 
\begin{lemma} [\cite{Devanur2011}]  \label{lem:expectedinstance} 
$\OPT_\dist \geq \OPT$ of the IID instance with distribution $\dist$. 
\end{lemma} 
We use the same notation $\OPT$ to mean different benchmarks, based on which stochastic model is under consideration.
}
\subsection{Performance Measures.} While the standard measure in competitive analysis of online algorithms is a multiplicative error w.r.t the benchmark, we mostly adopt a concept of additive error that is common in online learning, called the {\em regret}. 
Since we make no assumptions about $f$, it could even be negative, so an additive error is more appropriate. 
For certain special cases where multiplicative errors or competitive ratios are more natural or desirable, 
we discuss how our algorithms and analysis can be adapted to get such guarantees. 
We define the following two (average) regret measures, one for the objective and another for the constraint.\footnote{\addedShipra{In online learning, the objective value is the sum of reward in every step, which scales with $T$, and the regret typically scales with $\sqrt{T}$. But in our formulation, the objective $f(\frac{1}{T} \sum_t \vplay)$ is defined over average observations, therefore, to be consistent with the popular terminology,  we call our regret `average regret'.}} Let \addedShipra{ $d(\cv, S)$ denote the distance of the vector $\cv$ from the set $S$,} w.r.t. a given norm $\|\cdot\| .$
\begin{eqnarray*}
\areg_1(T) & = & \OPT -f(\vplayavg), \text{   and  }\\
\areg_2(T) & = & d(\vplayavg, S).
\end{eqnarray*}
\subsection{Main Results.} 
We now state the most general result we prove in this paper. 
\begin{theorem} 	\label{th:cp}
There is an algorithm (Algorithm \ref{algo:cp})  that achieves the following regret guarantees for the Online Stochastic Convex Programming problem, in the \RP~model.
 	\begin{eqnarray*}
	\Ex[\areg_1(T)] & = & (Z+L)\cdot O\left(\sqrt{\tfrac{C}{T}}\right)\\
	\Ex[\areg_2(T)] & = & O\left(\sqrt{\tfrac{C}{T}}\right)
	\end{eqnarray*}
 	where $C$ depends on the norm $\|\cdot\|$ used for defining distance. For Euclidean norm, $C=d\log(d)$. For $L_{\infty}$ norm,  $C=\log(d)$. 
	\addedShipra{The parameter $Z$ captures the tradeoff between objective and constraints for the problem, its value is problem-dependent and is discussed in detail later in the text. $L$ is the Lipschitz constant for $f$ w.r.t. the same norm $\|\cdot \|$ as used to measure the distance.
	}
\end{theorem} 
\addedShipra{In the main text we provide more detailed result statements, which will also make clear the dependence of our regret bounds on the regret bounds available for online learning, and implications of using different norms. These regret bounds can also be converted to {\em high probability} results, \revision{with  an additional $\sqrt{\log T}$ factor in the regret. This extra factor comes from simply taking a union bound over all time steps. A more careful analysis could possibly get rid of this extra factor, as was shown in \citet{GuptaM14} in case of online linear programming. 
}

These bounds are optimal, and this follows easily from an easy modification of a  lower bound given by  \citet{AWY2009} for the online packing problem.}


We also consider the following interesting special cases. 
\paragraph{Feasibility problem:} 
In this case, there is no objective function $f$, and there is only the constraint given by the set $S$. 
The goal is to make sure that the average of the chosen vectors lies as close to $S$ as possible, i.e., minimize $d(\vplayavg, S)$. 

\paragraph{Linear objective:} 
In this case, we assume that each vector $\cv \in A_t$ has an associated reward  $r \in [0,1]$. 
The objective is to maximize the total reward while making sure that the average of the vectors lies in $S$. 
This can be thought of as the special case where the vector you get is $(\cv,r)$, and the 
constraint is only on the subspace defined by all coordinates of this vector except the last,
while the objective is just the sum (or linear function) of its last coordinates.

\paragraph{Online Packing/Covering LPs:} 
This is a well studied special case of linear objective. The packing constraints $\sum_t \vplay \le B {\bf 1}$ are equivalent to using constraint set $S$ of the form $\{\cv: 0 \leq \cv \leq \tfrac{B}{T} \ones \}$, 
where $\ones$ is the vector of all 1s and $B > 0 $ is some scalar. 
In this case, we also assume that the sets $A_t$ always contain the origin, which corresponds to the option of ``doing nothing". 
The covering constraints are obtained when $S$ is $\{ \cv:\cv \geq \tfrac{B}{T} \ones \}$. \newline\\
\addedShipra{For online packing, we provide the following tighter guarantee in terms of competitive ratio.}

\begin{theorem}
\label{th:packing}
\revision{
For online stochastic packing problem, Algorithm \ref{algo:packing} achieves a competitive ratio of $1-O(\epsilon)$ in the \RP~model, given any $\epsilon >0$ such that $\min\{B, T\OPT\} \ge \log(d)/\epsilon^2$. Further, the algorithm has fast per-step updates, and needs to solve a sample \LP~at most once.
}
\end{theorem}

\section{Preliminaries}
\subsection{Fenchel duality.} 
\label{sec:Fenchel} 
As mentioned earlier, our algorithms are primal-dual algorithms. For the online packing problem, the LP duality framework (which is very well understood) is sufficient but for general convex programs we need the stronger framework of Fenchel duality. Below we provide some background on this useful mathematical concept. Let $h$ be a convex function defined  on $[0,1]^d$. We define $h^*$ as Fenchel conjugate of $h$,
\EQ{-0.08in}{-0.05in}{h^*(\thetaV):=\max_{\y \in [0,1]^d} \{ \y \cdot \thetaV - h(\y)\}}
For a given norm $\|\cdot\|$, we denote by $\|\cdot\|_*$, the dual norm defined as:
$$\|\y\|_* = \max_{\x: \|\x\|\le 1} \x^T\y.$$
Suppose that at every point $\x$, every supergradient $\bs{g}_{x}$ of $h$ has bounded dual norm $||\bs{g}_x||_* \le L$. 
Then, the following dual relationship is known between $h$ and $h^*$.
\begin{lemma}
\label{lem:FenchelDuality} 
$h(\z) = \max_{||\thetaV||_* \le L} \{ \thetaV \cdot \z-h^*(\thetaV)\}.$
\end{lemma}

A special case is when $h(\x) = d(\x,S)$ for some convex set $S$. This function is $1$-Lipschitz with respect to norm $||\cdot||$ used in the definition of distance. 
In this case, $h^*(\thetaV) = h_S(\thetaV):=\max_{\y\in S} \thetaV\cdot \y$, and Lemma \ref{lem:FenchelDuality} specializes to the following derivation which also appears in \citet{blackwell2011}.
\EQ{-0.08in}{-0.05in}{
d(\x,S) = \max_{||\thetaV||_* \le 1}\{\thetaV \cdot \x - h_S(\thetaV)\}.}

\strongc{
\subsection{Strong convexity/Smoothness Duality.}
We first define strong convexity and smoothness. 
\begin{definition}
A function $h:{\cal X} \rightarrow \mathbb{R}$ is $\beta$-strongly convex w.r.t. a norm $||\cdot||$ if $\forall \x,\y \in {\cal X}, \z \in \partial h(\x),$
$$h(\y)-h(\x) \ge \z\cdot(\y-\x) + \frac{\beta}{2} ||\x-\y||^2.$$
Equivalently for any $\x,\y$ in the interior of ${\cal X}$, and all $\alpha \in (0,1)$, we have that 
\begin{eqnarray*}
h(\alpha \x + (1-\alpha) \y) & \ge & \alpha h(\x) + (1-\alpha) h(\y) \\
& & \ - \frac{\beta}{2} \alpha(1-\alpha) ||\x-\y||^2.
\end{eqnarray*}
A function $h$ is $\beta$-strongly concave if and only $(-h)$ is $\beta$-strongly convex.
\end{definition}
\begin{definition}
A function $h:{\cal X} \rightarrow \mathbb{R}$ is $\beta$-strongly smooth w.r.t. a norm $||\cdot||$ if $h$ is everywhere differentiable, and for all $\x,\y \in {\cal X}$, we have
$$ \forall \x,\y \in {\cal X}, |h(\y)-h(\x) -\nabla h(\x) \cdot(\y-\x)| \le \frac{\beta}{2} ||\x-\y||^2.$$
\end{definition}
The following lemma can be derived from the proof of Theorem 6 in \cite{Kakade2009}. A proof is given in Appendix \ref{app:prelims:strongc} for completeness.
\begin{lemma}
\label{lem:strongc}
If $h$ is convex and $\beta$-strongly smooth with respect to norm $\|\cdot\|$, then $h^*(\thetaV)=\max_{\x \in [0,1]^d} \{\thetaV\cdot \x - h(\x)\}$ is $\frac{1}{\beta}$-strongly convex with respect to norm $||\cdot||_*$ on domain $\nabla_h = \{\nabla h(\x): \x \in [0,1]^d\}$.
\end{lemma}
}

\subsection{Online Learning.}\label{sec:oco}
A well studied problem in online learning, called the Online Convex Optimization (\OCO) problem, 
considers a $T$ round game played between a learner and an adversary (nature), 
where at round $t$, the player chooses a $\thetaV_t \in W$, and then the adversary picks a concave function $g_t(\thetaV_t): W \rightarrow \mathbb{R}$. The player's choice $\thetaV_t$ may only depend on the adversary's choices in the previous rounds. The goal of the player is to minimize regret defined as the difference between the player's objective value and the value of the best single choice in hindsight:
$$\regOCO(T):= \max_{\thetaV \in W}\sum_{t=1}^T g_t(\thetaV) -\sum_{t=1}^T g_t(\thetaV_t)$$ 
Some popular algorithms for \OCO~are online mirror descent (\OMD) algorithm and online gradient descent, which have very fast per step update rules, and provide the following regret guarantees. More details about these algorithms and their regret guarantees are in Appendix \ref{app:oco}.
\begin{lemma}{\cite{Shalev-Shwartz12}}
\label{lem:regOCO}
There is an algorithm for the \OCO~problem that achieves regret
 $$\regOCO(T) = O(G\sqrt{DT}),$$
where $D$ is the diameter of $W$ and $G$ is an upper bound on the norm of gradient of $g_t(\thetaV)$ for all $t$. The value of these parameters are problem specific. 
\end{lemma}
\revision{
In particular, following corollary can be derived, which will be useful for our purpose. Details are in Appendix \ref{app:oco}.
\begin{corollary}
\label{cor:regOCO}
For $g_t(\thetaV)$ of form $g_t(\thetaV)=\thetaV\cdot \z - h^*(\thetaV)$ and $W=\{\thetaV: ||\thetaV||_*\le L\}$, where $h$ is an $L$-Lipschitz function, \OCO~algorithms achieve regret bounds of ${\cal R}(T)\le O(L\sqrt{dT})$ for Eucledian norm, and $O(L\sqrt{\log(d)T})$ for $L_{\infty}$.
\end{corollary}
}
\noindent For optimization over a simplex, the {\em multiplicative weight update} algorithm is very fast and efficient: the step $t$ update of this algorithm  takes the following form, given that $0\le g_t(\thetaV_t)\le M$ and a parameter $\epsilon > 0$, 
\begin{equation}
\label{eq:MWupdate}
\thetaV_{t+1,j} = \frac{w_{t,j}}{\sum_j w_{t,j}}, \text{ where } w_{t,j} = w_{t-1,j}(1+\epsilon)^{g_t({\bf e}_j)/M}.
\end{equation}
The algorithm has the following stronger guarantees. 
\begin{lemma}{\cite{AHK12}}
\label{lem:regMW}
For domain $W=\{||\thetaV||_1 = 1, \thetaV\ge 0\}$,  given that $0\le g_t(\thetaV_t)\le M$, and for all $\epsilon> 0$, using the multiplicative weight update algorithm we obtain that 
for any $\thetaV\in W$,
$$\sum_{t=1}^T g_t(\thetaV_t) \ge (1-\epsilon) \left( \sum_{t=1}^T g_t(\thetaV)\right) -\frac{M \ln(d+1)}{\epsilon},$$
\end{lemma}

\strongc{
For strongly concave functions, even stronger {\it logarithmic} regret bounds can be achieved. 
	\begin{lemma}{\cite{HazanLogarithmic}}
		\label{lem:OCOstrongc}
		Suppose that $g_t$ is $H$-strongly concave for all $t$, and  $G \ge 0$ is an upper bound on the norm of the gradient, i.e. $\|\nabla g_t(\thetaV)\| \le G$, for all $t$.  Then the online gradient descent algorithm achieves the following guarantees for \OCO: for all $T\ge 1$,
		$$\regOCO(T) \le \frac{G^2}{H} \log(T).$$
	\end{lemma}
}

\section{Feasibility Problem}\label{sec:onlyS}
It will be useful to first illustrate our algorithm and proof techniques for the special case of the feasibility problem. 
In this special case of online stochastic \CP, there is no objective function $f$, and the aim of the algorithm is to have $\vplayavg$ be in the set $S$. 
The performance of the algorithm is measured by the distance from the set $S$, i.e., $d( \vplayavg,S)$. 
We assume that the instance is {\em always feasible}, i.e.,  there exist $\vopt \in A_t~ \forall~t$ such that $\tfrac{1}{T} \sum_{t=1}^T \vopt \in S$. 

\addedNikhil{The basic idea behind our algorithm is as follows. Suppose that instead of minimizing a convex function such as $d( \vplayavg,S)$ we had to minimize a linear function such as  $\thetaV \cdot \vplayavg $. 
This would be extremely easy since the problem then separates into small subproblems where at each time step 
we can simply solve $\min_{\vplay\in A_t} \thetaV\cdot \vplay$.  
In fact, convex programming duality guarantees exactly this -- that there is a $\thetaV^*$, 
such that an optimal (i.e., feasible) solution is $\vopt =\arg \min_{\cv \in A_t} \thetaV^*\cdot \cv$, however, we don't know $\thetaV^*$. 
This is where online learning comes into play. Online learning algorithms can provide a $\thetaV_t$ at every time $t$ using only the observations before time $t$, which together provide a good approximation to the best $\thetaV$ in hindsight.}

\begin{algorithm}[Feasibility problem]
\label{algo:onlyS}
  \begin{algorithmic}
	\STATE
Initialize $\thetaV_{1}$. 
\FORALL{$t=1,..., T$} 
		\STATE Set $\vplay =\arg \min_{\cv \in A_t}  \thetaV_t \cdot \cv$
		\STATE Choose $\thetaV_{t+1}$ by doing an \OCO~update with $g_t(\thetaV) =\thetaV \cdot \vplay - h_{S}(\thetaV)$, and domain $W=\{||\thetaV||_* \le 1\}$.
\ENDFOR

  	\end{algorithmic}
\end{algorithm}
 
Here $||\cdot||_*$ is the dual norm of $||\cdot||$, the norm used in the distance function. The updates required for selecting $\thetaV_{t+1}$, given $\thetaV_t$ and $g_t(\cdot)$, are given as Equation \ref{eq:OMDupdate} and Equation \ref{eq:MWupdate} for \OMD~and multiplicative weight update algorithm, respectively. As discussed there, these updates are simple and fast, and do not require solving any complex optimization problems.

\begin{theorem}
\label{th:onlyS}
Algorithm \ref{algo:onlyS} achieves the following regret bound for the {\em Feasibility Problem} in the \RP~model of stochastic inputs:
\begin{eqnarray*}
\Ex[\areg_2(T)] & := & \Ex[d( \vplayavg, S)] \\
& \le & O\left(\frac{\regOCO(T)}{T} + ||{\bf 1}_d||\sqrt{\frac{s\log(d)}{T}}\right).
\end{eqnarray*}
\end{theorem} 
Here $\regOCO(T)$ denotes the regret for \OCO~
with functions $g_t(\theta)$ and domain $W$, as defined in Section \ref{sec:oco}. 
And, $s\le 1$ is the coordinate-wise largest value a vector in $S$ can take. This parameter can be used to obtain tighter problem-specific bounds. 

\begin{proof}
From Fenchel duality, and by \OCO~guarantees,
\begin{eqnarray*}
d( \vplayavg, S) & = & \max_{||\thetaV||_*\le 1} \thetaV\cdot \vplayavg - h_S(\thetaV) \\
& = & \max_{||\thetaV||_*\le 1} \frac{1}{T} \sum_t g_t(\thetaV) \\
& \le & \frac{1}{T} \sum_t g_t(\thetaV_t) + \frac{1}{T}\regOCO(T).
\end{eqnarray*}
In Lemma \ref{lem:gtRP}, we upper bound $\Ex[\frac{1}{T}\sum_t g_t(\theta_t)]$  to obtain the statement of the theorem. 
\end{proof}

\begin{lemma} 
\label{lem:gtRP}
$\Ex[\sum_t g_t(\theta_t)]\le O(||{\bf 1}_d||\sqrt{s\log(d)T})$, where $s=\max_{\cv\in S} \max_j \cvS_{j} \le 1$, and $||\cdot||$ is the norm used in the distance function.
\end{lemma}
\begin{proof}

Let ${\cal F}_{t-1}$ denote the observations and decisions until time $t-1$. Note that $\thetaV_t$ is completely determined by ${\cal F}_{t-1}$.
Let $\cv_{X_t}$ denote the option chosen to satisfy request $X_t$ by the offline optimal (feasible) solution, and let $\vopt=\cv_{A_t}$. Then, since $A_t=X_{s}$, for $s=1,\ldots, T$ with equal probability, we have that $\Ex[\cv^*_{t}] = \frac{1}{T} (\cv_{X_1} + \ldots + \cv_{X_T}) \in S$. Therefore, due to the manner in which $\vplay$ was chosen by the algorithm, we have that
\begin{eqnarray}
\Ex[g_t(\thetaV_t) | {\cal F}_{t-1}] &=& \Ex[\thetaV_t\cdot \vplay  - h_{S}(\thetaV_t) |{\cal F}_{t-1}]\nonumber\\
& \le & \Ex[\thetaV_t\cdot \cv^*_{t} - h_{S}(\thetaV_t) | {\cal F}_{t-1}] \nonumber \\
& = & \thetaV_t\cdot \Ex[\cv^*_{t}] - h_{S}(\thetaV_t) \nonumber\\
& & + \ \thetaV_t\cdot(\Ex[\cv^*_{t} | {\cal F}_{t-1}] - \Ex[\cv^*_{t}]) \nonumber
\end{eqnarray}
\revision{Now, by the Fenchel dual representation of distance, for any $\cv, \thetaV'$ such that $\|\thetaV'\|_*\le 1$, $d(\cv, S) = \max_{||\thetaV||_*\le 1} \thetaV\cdot \cv - h_S(\thetaV) \ge \thetaV'\cdot \cv - h_S(\thetaV')$. Using this observation along with $\Ex[\cv^*_t]\in S$, we obtain from above,}
\begin{eqnarray}
\label{eq:gt}
\Ex[g_t(\thetaV_t) | {\cal F}_{t-1}] & \le & d(\Ex[\cv^*_t], S) + \thetaV_t\cdot(\Ex[\cv^*_{t} | {\cal F}_{t-1}] - \Ex[\cv^*_{t}])\nonumber\\
& =  & 0+\thetaV_t\cdot(\Ex[\cv^*_{t} | {\cal F}_{t-1}] - \Ex[\cv^*_{t}]) \nonumber\\
& \le & \|\Ex[\cv^*_{t} | {\cal F}_{t-1}] - \Ex[\cv^*_{t}]\|,
\end{eqnarray}
where the last inequality used the condition $\|\thetaV_t\|_*\le 1$.

Note that under independence assumption (\IID~model), we would have $\Ex[\cv^*_{t} | {\cal F}_{t-1}] = \Ex[\cv^*_{t}]$, so that the above inequality would suffice to give the required bound. However, in random permutation (\RP) model, the observations till time $t-1$ restrict the set of possible permutations.
Conditional on realization $A_1=X_{\pi(1)}, \ldots, A_{t-1}=X_{\pi(t-1)}$ until time $t-1$, for a given ordering $\pi$, we have that $A_t$ is one of the {\em remaining sets} with equal probability. 
So, $\Ex[\cv^*_{t} | {\cal F}_{t-1}] = \frac{1}{T-t+1} (\cv_{X_{\pi(t)}} + \ldots + \cv_{X_{\pi(T)}})$, for any ordering $\pi$ that agrees with ${\cal F}_{t-1}$ on the first $t-1$ indices. 

Next, we bound the gap $\|\Ex[\cv^*_{t} | {\cal F}_{t-1}]-\Ex[\cv^*_t]\|$ under random permutation assumption. 
For any given ordering $\pi$, define ${\bf w}_{t,\pi} =\frac{\cv_{X_{\pi(1)}} + \ldots + \cv_{X_{\pi(t)}}}{t}$. 
Also, for given ordering $\pi$, define $\pi'$ as the reverse ordering. 
Then, $\Ex[\cv^*_t | {\cal F}_{t-1}] = {\bf w}_{T-t+1,\pi'}$, for any ordering $\pi$ that agrees with ${\cal F}_{t-1}$ on the first $t-1$ indices. 
Now, the input ordering $\pi$ observed by the algorithm agrees with all the filtrations ${\cal F}_1, \ldots, {\cal F}_{T-1}$, and therefore taking $\pi'$ as the reverse of this ordering, we have that 
\begin{eqnarray*}
\sum_{t=1}^T \|\Ex[\cv^*_{t} | {\cal F}_{t-1}]-\Ex[\cv^*_t]\| & = & \sum_{t=1}^T \|{\bf w}_{T-t+1,\pi'} - \Ex[\cv^*_t]\|\\
& = &\sum_{t=1}^T \|{\bf w}_{t,\pi'} - \Ex[\cv^*_t]\|
\end{eqnarray*}
Due to the random permutation assumption, the input ordering $\pi$, and hence the reverse ordering $\pi'$ in above, is a uniformly random permutation. Also, taking expectation over uniformly random permutations $\sigma$, $\Ex[{\bf w}_{t,\sigma}] = \frac{(\cv_{X_1} + \ldots + \cv_{X_T})}{T} = \Ex[\cv_t^*]$. 
And, therefore,
\begin{eqnarray}
\label{eq:highProbW}
\sum_{t=1}^T \left\|\Ex\left[\cv^*_{t}  \condition {\cal F}_{t-1}\right] - \Ex[\cv^*_t]\right\|  & = & \sum_{t=1}^T \left\|{\bf w}_{t,\pi} -\Ex[{\bf w}_{t,\sigma}]\right\|\nonumber\\
\end{eqnarray}
where $\pi$ is a uniformly random permutation.
Taking outer expectations, and using \eqref{eq:gt}, this implies,
\begin{eqnarray*}
\Ex[\sum_t g_t(\thetaV_t)] & \le & \Ex\left[\sum_t \left\|\Ex\left[\cv^*_{t}  \condition {\cal F}_{t-1} \right] - \Ex[\cv^*_t]\right\|\right] \\
& = & \Ex\left[\sum_t \|{\bf w}_{t,\pi} -\Ex[{\bf w}_{t,\sigma}]\|\right].
\end{eqnarray*}

\addedShipra{Observe that for uniformly random permutation $\pi$, ${\bf w}_{t,\pi}$ can be viewed as the average of $t$ vectors sampled uniformly 
{\em without replacement} from the ground set $\{\cv_{X_{1}}, \ldots, \cv_{X_{T}}\}$ of $T$ vectors. We use Chernoff-Hoeffding type concentration bounds for sampling without replacement (refer to Appendix \ref{app:onlyS} for details), to obtain,
\begin{equation}
\label{eq:swr}
\Ex[||{\bf w}_{t,\pi} - \Ex[ {\bf w}_{t,\sigma}]||] \le O(||{\bf 1}_d|| \sqrt{\frac{ s\log(d)}{t}}).
\end{equation}
The lemma statement then follows by summing up these bounds over all $t$. 
}
\end{proof}

\begin{remark}{\em [\RP~vs. \IID]}
For the IID model,  since $\Ex[\cv^*_{t} | {\cal F}_{t-1}] = \Ex[\cv^*_{t}]$, we would get $\sum_t \Ex[\g_t(\thetaV_t)] \leq 0$ directly from Equation \eqref{eq:gt}. Thus, the quantity $\Ex[\sum_t \|\Ex[\cv^*_{t} | {\cal F}_{t-1}] - \Ex[\cv^*_{t}]\|] \le  O(\|{\bf 1}_d\| \sqrt{sT\log(d)})$ characterizes the gap between \IID~and \RP~models. 
\end{remark}
\begin{remark}\label{rem:highProb}{\em [High probability bounds]}
The above analysis can be extended to bound the sum  
 of {\em conditional expectations} $\sum_t \Ex[\g_t(\thetaV_t) |{\cal F}_{t-1}] \le \sum_t \|{\bf w}_{t,\pi} - \Ex[w_{t,\sigma}]\|$ 
by $O(\|{\bf 1}_d\| \sqrt{T\log(dT/\rho)})$ with high probability $1-\rho$.  As a result, we obtain a high probability regret bound of $O(\|{\bf 1}_d\|\sqrt{\frac{\log(Td)}{T}})$. Details are in Appendix \ref{app:onlyS}. For the \IID~model, this sum of conditional expectations is bounded by $0$, so the resulting high probability bounds are slightly stronger, with no extra $\sqrt{\log(T)}$ factor. 
\end{remark}


\renewcommand{\dist}{{}}
\section{Online stochastic convex programming}
\label{sec:cp}
In this section, we extend the algorithm from previous section to the general online stochastic Convex Programming (\CP) problem, as defined in Section \ref{sec:prelims}. Recall that the aim here is to maximize $f( \vplayavg)$ while ensuring $ \vplayavg \in S$. 

A direct way to extend the algorithm from the previous section would be to reduce the convex program to the feasibilty problem with constraint set $S'=\{\cv: f(\cv)\ge \OPT, \cv \in S\}$. However, this requires the knowledge of $\OPT$. If $\OPT$ is estimated, the errors in the estimation of $\OPT$ at all time steps $t$ would add up to the regret, thus this approach would tolerate very small $\tilde{O}(\frac{1}{\sqrt{t}})$ per step estimation errors. In this section, we propose an alternate approach of combining objective value and distance from constraints using a parameter $Z$, which will capture the tradeoff between the two quantities.
We may still need to estimate this parameter $Z$, however, $Z$ will appear only in the second order regret terms, so that a constant factor approximation of $Z$ will suffice to obtain optimal order of regret bounds. This makes the estimation task relatively easy and  enable us to get better problem specific bounds. 
As a specific example, for the online packing problem, we can use $Z=\frac{\OPT}{(B/T)}$  so this approach requires only a constant factor approximation of $\OPT$ and  the resulting algorithm obtains the optimal competitive ratio. 
(See \prettyref{sec:packing} for more details.)


To illustrate the main ideas in our algorithm, let us start with the following assumption. 
\begin{assumption}
\label{assum:Z}
Let $\OPT^{\delta}$ denote the optimal value 
of the offline problem that maximizes $f(\frac{1}{T} \sum_t \cv_t)$ with feasibility constraint relaxed to $ d(\frac{1}{T}\sum_t \cv_t, S) \le \delta$. We are given a $Z \ge 0$ such that that for all $\delta \ge 0$, 
\begin{equation}
\label{eq:Zprop}
\OPT^{\delta}\le \OPT + Z \delta.
\end{equation}
\end{assumption}
In fact, such a $Z$ always exists, as shown by the following lemma.
\begin{lemma}
\label{lem:Z}
 $\OPT^\delta$ is a non-decreasing concave function of the constraint violation $\delta$, and its gradient at $\delta=0$ is the minimum value of $Z$ that satisfies the property \eqref{eq:Zprop}. This gradient is also equal to the value of the optimal dual variable corresponding to the distance constraint. 
\end{lemma}
The proof of this lemma is provided in Appendix \ref{app:lem:Z}. This fact is known for linear programs. 

Below, we present an algorithm (Algorithm \ref{algo:cp}) for online stochastic \CP~ assuming we are given parameter $Z$ as in Assumption \ref{assum:Z}.  This algorithm is based on the same basic ideas as the algorithm for the feasibility problem in the previous section. Here, we linearize both objective and constraints using Fenchel duality, and estimate the corresponding dual variables using online learning as blackbox. And, we use parameter $Z$ to combine objective with constraints. The resulting algorithm has very efficient per-step updates and does not require solving a (sample) \CP~in any step, and we prove that it achieves the regret bound stated in Theorem \ref{th:cp}. 

The regret of this algorithm (as stated in Theorem \ref{th:cp}) scales with the value of $Z$, and it is desirable to use as small a value of $Z$ as possible. 
If such a $Z$ is not known, in Appendix \ref{app:estZ} we demonstrate how we can approximate the optimal value of $Z$ up to a constant factor by solving a logarithmic number of sample \CP s overall. 




\begin{algorithm}[Online convex programming]
\label{algo:cp}
  \begin{algorithmic}
	\STATE
 Initialize $\thetaV_{1}, \phiV_1$. 
\FORALL{$t=1,..., T$} 
		\STATE Choose option 
		\begin{center} $\vplay=\arg \max_{\cv \in A_t}   -\phiV_t \cdot \cv  - 2(Z+L) \thetaV_t \cdot \cv$.\end{center}
		\STATE Choose $\thetaV_{t+1}$ by doing an \OCO~update for $g_t(\thetaV) =\thetaV \cdot\vplay - h_{S}(\thetaV)$ over domain $W=\{\|\thetaV\|_*\le 1\}$.
		\STATE Choose $\phiV_{t+1}$ by doing an \OCO~update for 
		$\psi_t(\phiV) = \phiV \cdot\vplay-(-f)^*(\phiV)$ over domain $U=\{\|\phiV\|_*\le L\}$.
\ENDFOR

  	\end{algorithmic}
\end{algorithm}
A complete proof of Theorem \ref{th:cp}, along with a more detailed theorem statement, is provided in Appendix \ref{app:CP}.
Here, we provide the proof for the simpler case of {\em linear objective} discussed in Section \ref{sec:prelims}. In this setting,  each option in $A_t$ is associated with a reward $r$ in addition to the vector $\cv$. And, at every time step $t$, the player chooses $(\rplay, \vplay)$, in order to maximize $\frac{1}{T}\sum_t \rplay$ while ensuring $ \vplayavg \in S$. ( We will use $\rplayavg $ to denote $\frac{1}{T}\sum_t \rplay$.) 
The proof for this special case will illustrate the main ideas required for proving regret bounds for the online \CP~problems with `objective plus constraints', over and above the techniques used in the previous section for the case of `only constraints'. 

For this special case, Algorithm \ref{algo:cp} reduces to the following:

\begin{algorithm}[Linear objectives]
\label{algo:linear}
  \begin{algorithmic}
	\STATE  
	\STATE Initialize $\thetaV_{1}$. 
\FORALL{$t=1,..., T$} 
		\STATE Choose option\\
		\begin{center}$(\rplay, \vplay) =\arg \max_{(r,\cv) \in A_t}   r - 2Z \thetaV_t \cdot \cv.$\end{center}
		\STATE Choose $\thetaV_{t+1}$ by doing \OCO~update with $g_t(\thetaV) =\thetaV \cdot \vplay - h_{S}(\thetaV)$, and domain $W=\{\|\thetaV\|_* \le 1\}$.
\ENDFOR

  	\end{algorithmic}
\end{algorithm}
\begin{theorem}
\label{th:rPlusS}
Given $Z$ that satisfies Assumption \ref{assum:Z}, Algorithm \ref{algo:linear} achieves the following regret bounds for online stochastic \CP~with linear objective, in \RP~model:
\begin{eqnarray*}
\Ex[\areg_1(T)] & \le & \frac{Z}{T}\cdot O(\regOCO(T) + {\cal Q}(T)) \text{ and }\\
\Ex[\areg_2(T)] & \le & \frac{1}{T} \cdot O(\regOCO(T) + {\cal Q}(T)).
\end{eqnarray*}
Here, ${\cal Q}(T)=O(\oneNorm \sqrt{sT\log(d)})$, $s=\max_{\cv\in S} \max_j v_j$, and $\regOCO(T)$ denotes the \OCO~regret for $g_t(\cdot)$ over domain $W$. 
\end{theorem}
\begin{proof}
Denote by $(\ropt, \vopt)$ the choice made by the offline optimal solution to satisfy request $A_t$. Then,
$$ \Ex[\ropt] = \OPT,\text{ and } \Ex[\vopt] \in S,$$
where expectation is over $A_t$ drawn uniformly at random  from $X_1, \ldots, X_T$.
 
Lemma \ref{lem:gtRPreward} upper bounds $\sum_t \Ex[ 2Z g_t(\thetaV_t) - \rplay + \ropt]$ by $2Z{\cal Q}(T)=2ZO(||{\bf 1}_d||\sqrt{s\log(d)T})$, using exactly the same line of argument as the proof of Lemma \ref{lem:gtRP}. Therefore, using $\Ex[\ropt] =  \OPT$, the expected average reward obtained by the algorithm can be lower bounded as
$$\Ex[ \rplayavg] \ge \OPT + \frac{2Z}{T}\sum_t \Ex[ g_t(\thetaV_t)] - \frac{2Z}{T}{\cal Q}(T).$$
As in the proof of Theorem \ref{th:onlyS}, using Fenchel duality and \OCO~guarantees, it follows that
$d( \vplayavg, S) \le \frac{1}{T} \sum_t g_t(\thetaV_t) + \frac{1}{T}\regOCO(T), \text{ which gives,}$
\begin{equation}
\label{eq:key}
\Ex[ \rplayavg] \ge \OPT +(2Z) \Ex[d(\vplayavg, S)]-\tfrac{2Z}{T} \regOCO(T) - \tfrac{2Z}{T}{\cal Q}(T)
\end{equation}
Now, we use Assumption \ref{assum:Z} 
to upper bound the reward obtained by the algorithm in terms of \OPT~and distance from set $S$. In particular, for $\delta:=\Ex[d(\vplayavg, S)]$, since $d(\Ex[\vplayavg], S) \le \Ex[d(\vplayavg, S)] = \delta$,
\begin{equation}
\label{eq:Zapplication}
\Ex[ \rplayavg] \le \OPT^\delta_\dist \le \OPT_\dist + Z \delta = \OPT+Z \cdot \Ex[d( \vplayavg, S)].
\end{equation}
Combining inequalities \prettyref{eq:key} and \prettyref{eq:Zapplication}, we obtain
$$ \Ex[d( \vplayavg, S)] \le \tfrac{2}{T} \regOCO(T) + \tfrac{2}{T}{\cal Q}(T),$$
and from \eqref{eq:key}, using the fact that $\Ex[d(\vplayavg, S)]\ge 0$, we get that 
$$\Ex[\rplayavg] \ge \OPT-\tfrac{2Z}{T}\cdot \left(\regOCO(T)+ {\cal Q}(T)\right).$$
This gives the theorem statement.

\comment{\ASI~benchmark does not work out in the above argument. Here are the details.
For \IID~and \ASI~model, denote by $(\ropt, \vopt)$  the option chosen to satisfy request $A_t$ by the {\em Pure-random} algorithm. Recall that that the {\em Pure-random algorithm} is a non-adaptive algorithm that uses the optimal solution of the expected instance for distribution $\dist_t$ to satisfy request $A_t$ ($\dist_1 = \cdots =\dist_t = \dist$ for \IID). By this definition, we have that 
$$ \Ex_{A_t \sim \dist_t}[\ropt] \ge \OPT_{\dist_t},\text{ and } \Ex_{A_t \sim \dist_t}[\vopt] \in S$$
Then, using Lemma \ref{lem:gtIIDreward}, and above arguments, we will obtain,
\begin{equation}
\Ex[\frac{1}{T}\sum_t \rplay] \ge  \frac{1}{T} \sum_t \OPT_{\dist_t} +(2Z) \Ex[d(\frac{1}{T}\sum_t \vplay, S)]-\frac{2Z}{T} \regOCO(T) - \frac{2Z}{T}{\cal Q}(T)
\end{equation}
But, for upper bound, we can only get
\begin{equation}
\Ex[ \frac{1}{T}\sum_t\rplay] \le \max_{\dist_t} (\OPT^\delta_{\dist_t}) \le \max_t \OPT_{\dist_t} + Z_{\dist_t} \delta \le \OPT_{max}+Z \cdot \Ex[d(\frac{1}{T}\sum_t \vplay, S)].
\end{equation}
}


\end{proof}
\begin{lemma}
\label{lem:gtRPreward}
 $\Ex[\sum_t  2Z g_t(\theta_t) - \rplay + \ropt] \le O(Z \oneNorm \sqrt{sT\log(d)}).$
\end{lemma}
The proof of the above lemma follows exactly the same line of argument as the proof of Lemma \ref{lem:gtRP}. We omit it for brevity.

\newcommand{\ca}{c_1}
\newcommand{\cb}{c_2}
\newcommand{\cc}{c_3}
\newcommand{\cd}{c_4}
\newcommand{\ce}{c_5}
\newcommand{\opteps}{\OPTest}
\newcommand{\ej}{{\bf e}_j}
\section{Online stochastic packing }\label{sec:packing}
Recall that the {\em online stochastic packing}  problem is a special case of the online stochastic CP with linear objectives, with $S=\{\y: \y\le \frac{B}{T} {\bf 1}\}$. However, the performance of an  algorithm for online stochastic packing is typically measured by {\em competitive ratio}, which is the ratio of  total expected reward obtained by the online algorithm to the optimal solution or benchmark. 
 The benchmarks in online packing are defined as sum of rewards, where as we defined $\OPT$ as the average reward. Therefore, in our notation, the competitive ratio for the online packing problem is given by
 $\frac{\Ex[\sum_t\rplay]}{T\OPT} =  \frac{\Ex[\frac{1}{T}\sum_t\rplay]}{\OPT} .$
The competitive ratio we obtain is $1-O(\epsilon)$, for any $\epsilon >0$ such that $\min\{B, T\OPT\} \ge \log(d)/\epsilon^2$. 

Another important difference is that for online packing the budget is not allowed to be violated at all, 
while  online CP allows a small violation of the constraint. 
A simple fix to make sure that budgets are not violated is to simply stop whenever a budget constraint is breached.\footnote{Note that 
	such a stopping rule does not make sense for a general $S$. If $S$ is downwards closed, then one can consider similar stopping rules in those cases as well.}
Another change we make to the algorithm is that we use a slightly different function in the \OCO~algorithm. We will use 
\[ g_t(\thetaV) = (\vplay - \frac{B}{T} \ones)   \cdot \thetaV\]
over the domain $||\thetaV||_1 \leq 1, \thetaV \geq \mathbf{0}.$ 
This domain is the convex hull of all the basis vectors and the origin, therefore we can use 
the multiplicative weight update algorithm as our \OCO~algorithm, which provides strong guarantees (refer to Lemma \ref{lem:regMW}, here $M=1$). 

Finally, as with the previous algorithms, we state the algorithm assuming  we are given the parameter $Z$. We then show how to estimate $Z$ to desired accuracy using only an $O(\epsilon^2 \log (1/\epsilon) )$ fraction of samples and solving an LP {\em only once} (in Lemma \ref{lem:Zestimate}), 
\revision{assuming that $\min\{B, T\OPT\} \ge \frac{\log(d)}{\epsilon^2}$.}

We now state the algorithm below for the online stochastic packing problem: 

\begin{algorithm}[Online Packing]
	\label{algo:packing}
	\begin{algorithmic}
		\STATE
		\STATE Initialize $\thetaV_{1}= \tfrac 1 {d+1} \ones$. 
		\STATE Initialize $Z$ such that $\tfrac{T\OPT}{B} \leq Z\le O(1)\tfrac{T	\OPT}{B}$. 
		\FORALL{$t=1,..., T$} 
		\STATE Choose option 
		\begin{center}$(\rplay, \vplay) =\arg \max_{(r,\cv) \in A_t} \left \{ r - Z \thetaV_t \cdot \cv\right \}.$\end{center}
		\STATE If for some $j = 1..d, \sum_{t'\leq t} \cv^\dagger_{t'}\cdot {\bf e}_j \geq B $ then EXIT.  
		\STATE Update  $\thetaV_{t+1}$ using multiplicative weight update: 
		\begin{center}$\forall\ j = 1..d, w_{t,j} = w_{t-1,j}(1+\epsilon)^{\vplay\cdot\ej - B/T}  $\end{center}
		and 
		\begin{center}$\forall\ j = 1..d, \thetaV_{t+1,j} = \frac{w_{t,j}}{1+\sum_{j'=1}^d w_{t,j'}},$\end{center}
	  
		\ENDFOR
	\end{algorithmic}
\end{algorithm}

Strictly speaking, if we use the first few requests as samples to estimate $Z$, then we need to ignore these requests, and bound 
the error due to this. However,  since the number of samples required is only $O(\epsilon^2\log(1/\epsilon))$ fraction of all requests, this error is quite small relative to the guarantee we obtain, which is a competitive ratio of $1-O(\epsilon)$.
We therefore ignore this error for the ease of presentation.

Let $\tau$ be the stopping time of the algorithm.  Denote by $(\ropt, \vopt)$ the choice made by the offline optimal solution to satisfy request $A_t$. We begin with the following lemma which is similar to Lemma \ref{lem:gtRP}. 
\begin{lemma}\label{lem:OP1}
\begin{eqnarray*}
\sum_{t=1}^{\tau} \Ex[\rplay | {\cal F}_{t-1}]	& \ge & \tau \OPT +Z \sum_{t=1}^{\tau} \thetaV_t \cdot \Ex[\vplay-{\bf 1}\frac{B}{T}|{\cal F}_{t-1}] \\
& & \ - \sum_{t=1}^\tau Q(t)
\end{eqnarray*}
where $Q(t)=Z||\Ex[\vopt] - \Ex[\vopt | {\cal F}_{t-1}]||+|\Ex[\ropt] - \Ex[\ropt | {\cal F}_{t-1}]|$. 
\end{lemma}
\begin{proof}
	If $A_t$ is drawn uniformly at random  from $X_1, \ldots, X_T$, then 
	$ \Ex[\ropt] = \OPT,\text{ and } \Ex[\vopt] \le \frac{B}{T}{\bf 1}.$
	The algorithm chooses $(\rplay,\vplay) = \arg\max_{(r,\cv) \in A_t} r- Z(\thetaV_t \cdot \cv)$. 
	By the choice made by the algorithm
	\begin{eqnarray*}
		\rplay - Z (\thetaV_t \cdot \vplay) & \ge & \ropt - Z (\thetaV_t \cdot \vopt)
		\end{eqnarray*}
	 \begin{eqnarray*}
		\Ex[\rplay - Z (\thetaV_t \cdot \vplay) | {\cal F}_{t-1}] & \ge &  \Ex[\ropt | {\cal F}_{t-1}] \\
		& & \ - Z (\thetaV_t \cdot \Ex[\vopt | {\cal F}_{t-1}] )\\
		& \ge &  \Ex[\ropt] - Z (\thetaV_t \cdot \Ex[\vopt] ) \\
		& & \ - {Q}(t)\\
		& \ge &  \OPT  - Z \thetaV_t\cdot \frac{B{\bf 1}}{T}  - {Q}(t)
	\end{eqnarray*}
	Summing above inequality for $t=1$ to $\tau$ gives the lemma statement. 
\end{proof}
\begin{lemma}\label{lem:OP2}
	$$\sum_{t=1}^{\tau} \thetaV_t\cdot (\vplay-\frac{B}{T}{\bf 1}) \ge(1-\epsilon)(B-\frac{\tau B}{T}) - \frac{ \log(d+1)}{\epsilon }.$$ 
\end{lemma}
\begin{proof}
	
Recall that  $g_t(\thetaV_t) = \thetaV_t \cdot \left(\vplay -\frac{B}{T}{\bf 1}\right)$, therefore the 
LHS in the required inequality is $\sum_{t=1}^{\tau}g_t(\thetaV_t)$.  Let 
$\thetaV^* := \arg \max_{||\thetaV||_1\le 1, \thetaV\ge 0}\sum_{t=1}^{\tau} g_t(\thetaV) $. 
We use the regret bounds for the multiplicative weight update algorithm given in Lemma \ref{lem:regMW},  
to get that 
$\sum_{t=1}^{\tau} g_t(\thetaV_t) \ge (1-\epsilon) \sum_{t=1}^{\tau}  g_t(\thetaV^*)-  \tfrac{ \log(d+1)}{\epsilon }.$

Now either $\sum_{t=1}^\tau (\vplay \cdot {\bf e}_j)\geq B$ for some $j$ at the stopping  time $\tau$, so that $\sum_{t=1}^{\tau} g_t(\thetaV^*) \ge \sum_{t=1}^{\tau} g_t({\bf e}_j) \geq B-\frac{\tau B}{T}$. Or, $\tau = T, \sum_{t=1}^\tau (\vplay)_j <B$ for all $j$, in which case, the maximizer is $\thetaV^*={\bf 0}$. Therefore we have that  
$ \sum_{t=1}^{\tau}  g_t(\thetaV^*)\ge B-\tfrac{\tau B}{T},$ 
which completes the proof of the lemma. 
\end{proof}

Now, we are ready to prove Theorem \ref{th:packing}, which states that Algorithm \ref{algo:packing} achieves a competitive ratio of $1-O(\epsilon)$, given $\min\{B,T\OPT\} \ge \frac{\log(d)}{\epsilon^2}$ 
for the online stochastic packing problem in \RP~model. \newline\\
\noindent {\it \bf Proof of Theorem \ref{th:packing}.} 
Substituting the inequality from Lemma \ref{lem:OP2} in Lemma \ref{lem:OP1}, we get 
\begin{eqnarray*}
	\sum_{t=1}^{\tau} \Ex[\rplay| {\cal F}_{t-1}]	
	& \ge & \tau \OPT+ (1-\epsilon)ZB \left(1-\frac{\tau}{T}\right) \\
	& & \ - Z \frac{ \log(d+1)}{\epsilon } - \sum_{t=1}^\tau {Q}(t) 
	\end{eqnarray*}
Now, using $ Z \leq  O(1) \tfrac{T	\OPT}{B} $ and $B \ge \frac {\log(d)}  {\epsilon^2} $, we get
\[      Z \frac{ \log(d+1)}{\epsilon } \leq O(1)\frac{T	\OPT}{B}  \frac{ \log(d+1)}{\epsilon } = O(\epsilon ) T \OPT    . \]
Also, $Z\ge \frac{T\OPT}{B}$. Substituting in above,
	\begin{eqnarray*}
	\sum_{t=1}^{\tau} \Ex[\rplay| {\cal F}_{t-1}]	& \ge & (1-\epsilon)\tau\OPT + (1-\epsilon) \OPT( T - \tau)  \\
	& & \ -O(\epsilon )T\OPT- \sum_{t=1}^\tau {Q}(t) \\
	& \ge & (1-O(\epsilon))T\OPT  - \sum_{t=1}^\tau {Q}(t) 
\end{eqnarray*}
Then, taking expectation on both sides,
$
	\Ex[\sum_{t=1}^{\tau} \rplay]	 \ge (1-O(\epsilon))T\OPT - \Ex[\sum_{t=1}^\tau {Q}(t)] . 
$

Just like in the proof of Lemma \ref{lem:gtRP}, we can  bound 
$\Ex[\sum_{t=1}^\tau {Q}(t)]\leq Z||{\bf 1}_{d+1}||_{\infty}\sqrt{sT\log(d+1)}$ which is $O(\epsilon) T\OPT$,
using the fact that  for $S=\{\y: \y\le \frac{B}{T} {\bf 1}\}$, the parameter $s = \max_{j, \y\in S} y_j = \frac{B}{T}$, $||{\bf 1}_{d+1}||_\infty= 1$, and that $Z\le O(1) \frac{T\OPT}{B}, \epsilon \ge \sqrt{\frac{\log(d)}{B}}$. 
This completes the proof.  


\comment{The same derivation applies to \IID~case with $(\ropt, \vopt)$ being the choice made by the {\em Pure-random} non-adaptive algorithm, to satisfy request $A_t$. Then,
$ \Ex[\ropt] \geq \OPT,\text{ and } \Ex[\vopt] \le \frac{B}{T}{\bf 1}.$
Then, we can repeat the above analysis. In the \IID~case $\Ex[\vopt|{\cal F}_{t-1}]= \Ex[\vopt], \Ex[\ropt |{\cal F}_{t-1}] = \Ex[\ropt]$, so that $Q(t)=0$.
In fact,  since $Q(t)$ is always 0, and we can actually bound the sum of the conditional expectations, $\sum_t \Ex[\rplay |{\cal F}_{t-1}] $, we can  use Chernoff bounds to get a high probability result for the \IID~model. 
}

\newcommand{\opthat}{\hat{\OPT}}
\newcommand{\optsum}{\OPT_{\text{\sc sum}}}
We now show how to compute a $Z$ as required using the first $O(\epsilon^2\log (1/\epsilon))$ requests as samples. 
For convenience, let $\optsum:= T\OPT $ denote the optimum for the sum. 
We first state a lemma that relates the optimum value of an offline packing instance to the optimum value on a sample of the requests. The proof of this is along the lines of a similar lemma (Lemma 14) in \cite{DJSWfull}, and  we present the proof in Appendix \ref{app:optestimation} for the sake of completeness. 

\begin{lemma}\label{lem:optestimation}
	For all $\rho \in (0,1]$, there exists $\eta = O\left( \sqrt{\log(\tfrac {d}{\rho} )} \right)$ such that for all $\delta \in (0,1] $, given a random sample of $\delta T $ requests, one can compute a quantity $\opthat$ such that with probability $1-\rho$, 
	\begin{enumerate}
		\item $\opthat \geq \optsum - \eta \sqrt{\optsum/\delta}.$ 
		\item $\frac \opthat {1 + \eta/\sqrt{\delta B}} \leq \optsum + \eta \sqrt{\optsum/\delta}.$ 
	\end{enumerate}
\end{lemma}

\begin{lemma}\label{lem:Zestimate}
	Given a random sample of $O(\epsilon^2 \log (1/\epsilon))$ fraction of requests, one can compute a quantity $Z$ such that with probability at least $1-\epsilon^2$, 
	$$\frac \optsum B \leq Z \leq \frac 9 2 \frac \optsum B.$$ 
\end{lemma}

\begin{proof}
	We use \prettyref{lem:optestimation} with $\rho = \epsilon^2$ and $\delta = 4\eta^2\epsilon^2/\log(d)
	$. Then, from the assumption that  $\min\{ B, \optsum\} \geq \log(d)/\epsilon^2$, we have that $\delta \ge 4\eta^2/\optsum$, and $\delta \ge 4\eta^2/B$. Therefore, we get that with probability at least $1-\epsilon^2$,
	\begin{eqnarray*}
	\opthat & \geq & \optsum - \eta \sqrt{\optsum/\delta} \\
	& \geq & \optsum -  \optsum/2 =  \optsum/2.
	\end{eqnarray*}
	Also, 	
	\begin{eqnarray*} 
	\opthat & \leq & ({1 + \eta/\sqrt{\delta B}} )(\optsum + \eta \sqrt{\optsum/\delta}) \\
	& \leq & \frac 3 2 (\optsum + \frac 1 2 \optsum )\\ 
	& \leq & \frac 9 4 \optsum. 
	\end{eqnarray*}
	Therefore $Z:= 2 \opthat/B$ satisfies the conclusion of the lemma. Finally, note that  $\delta  = 4\eta^2\epsilon^2/\log(d) =  O\left( \epsilon^2\log(\tfrac {d}{\epsilon} ) /\log(d) \right ) = O(\epsilon^2 \log (1/\epsilon))$ . 
\end{proof}

\section{Stronger bounds for smooth functions}
We show that when $f$ is a strongly smooth function, and, instead of distance function a strongly smooth function is used to measure regret in constraint violation, then stronger regret bounds of $\tilde{O}(\frac{\log T}{T})$ can be achieved in IID case. Intuitively, this is because as discussed in Section \ref{sec:prelims}, the dual of strongly smooth functions is strongly convex, and for strongly convex/concave functions, stronger logarithmic regret 
guarantees are provided by online learning algorithms. 

More precisely, consider the following {\em smooth} version of Online Convex Programming problem.

\begin{definition}{\it \scshape [Online Stochastic Smooth Convex Programming]}
Let $f$ be a $\beta$-smooth concave function. And, 
let $h$ be a $\beta$-smooth convex function 
At time $t$, the algorithm needs to choose $\vplay \in A_t$ to minimize regret defined as
\begin{eqnarray*}
\areg_1(T) & := & f(\voptavg)-f(\vplayavg),\\ 
\areg_2(T) & := & h(\vplayavg).
\end{eqnarray*}
Here, $\voptavg=\frac{1}{T} \sum_{t=1}^T \vopt, \vplayavg=\frac{1}{T} \sum_{t=1}^T \vplay$. Also, assume that there exist $\cv_t \in A_t$ for all $t$, such that $h(\frac{1}{T} \sum_t \cv_t)=0$.
\end{definition}
Note that we do not require Lipschitz condition for $f$ or $h$. 
We make an additional assumption. 
\begin{assumption}
\label{assum:grad}
Let $\nabla_f$ and $\nabla_g$ denote the set of gradients of functions $f$ and $g$, respectively, on domain $[0,1]^d$, i.e.,
\begin{eqnarray*}
\nabla_f & = & \{\nabla f(\x): \x\in [0,1]^d\}, \text{ and },\\
\nabla_g & = & \{\nabla g(\x): \x\in [0,1]^d\}.
\end{eqnarray*}
Assume that the sets $\cl(\nabla_f)$ and $\cl(\nabla_g)$ are convex and easy to project upon. Here $\cl(S)$ denotes the closure of set $S$.
\end{assumption}
This assumption is true for many natural concave utility and convex risk functions, in particular, for all separable smooth functions. 
Now, an algorithm similar to Algorithm \ref{algo:cp} can be used for this problem. One change we make is that we perform online learning for $g_t$ and $\psi_t$ on domain $\nabla_g$ and $\nabla_f$, respectively, which is possible because from Assumption \ref{assum:grad}, these domains are convex and easy to project upon.\\
\begin{algorithm}[Online Smooth CP]
\label{algo:cpSmooth}
  \begin{algorithmic}
	\STATE
	\STATE  Initialize $\thetaV_{1}, \phiV_1$. 
\FORALL{$t=1,\ldots, T$} 
		\STATE Choose vector 
		\begin{center}$\vplay=\arg \max_{\cv \in A_t}   -\phiV_t \cdot \cv  - 2Z \thetaV_t \cdot \cv.$\end{center}
		\STATE Choose $\thetaV_{t+1}$ by doing an \OCO~update for $g_t(\thetaV) =\thetaV \cdot\vplay - h_{S}(\thetaV)$ over domain $\nabla_g$.
		\STATE Choose $\phiV_{t+1}$ by doing an \OCO~update for 
		$\psi_t(\phiV) = \phiV \cdot\vplay-(-f)^*(\phiV)$ over domain $\nabla_f$.
\ENDFOR
\end{algorithmic}
\end{algorithm}
\begin{theorem} 	\label{th:cpSmooth}
Under Assumption \ref{assum:grad}, and given $Z$ that satisfies Assumption \ref{assum:Z}, Algorithm \ref{algo:cpSmooth} achieves the following regret for the  Online {\em Smooth} Convex Programming problem, in the stochastic \IID~input model.
 	\begin{eqnarray*}
	\Ex[\areg_1(T)] & = & Z \cdot O\left(\tfrac{C \log(T)}{T}\right), \\ 
	\Ex[\areg_2(T)] & = & O\left(\tfrac{C \log(T)}{T}\right),
	\end{eqnarray*}
 	where $C=\beta ||{\bf 1}_d||^2$. 
\end{theorem}
\begin{proof}
The proof follows from the proof of Theorem \ref{th:cp} on observing that stronger  \OCO~regret bounds of $O(\log(T))$ are available for strongly convex functions. 
More precisely, in case of IID inputs, 
the proof of Theorem \ref{th:cp} can be followed as it is to achieve the following regret bounds. (These are same as in the detailed statement of Theorem \ref{th:cp}, provided in Appendix \ref{app:CP}, but with ${\cal Q}(T)=0$ due to \IID~assumption.)
$$ \Ex[\areg_1(T)] \le \frac{Z}{T}\cdot O(\regOCO(T)) + O(\frac{\regOCO'(T)}{T}),$$
$$ \Ex[\areg_2(T)] \le \frac{1}{T} \cdot O\left(\regOCO(T) \right) + \frac{1}{Z} O(\frac{\regOCO'(T)}{T}),$$
Here $\regOCO(T)$ is $\OCO$~regret for the problem of maximizing concave function $g_t(\thetaV) = \thetaV\cdot \cv_t - h^*(\thetaV)$, $\regOCO'(T)$ is $\OCO$~regret for the problem of maximizing concave function $\psi_t(\phiV) = \phiV\cdot \cv_t - (-f)^*(\phiV)$. Now, using Lemma \ref{lem:strongc}, given that $h$ and $f$ are $\beta$-strongly smooth, $g_t$ and $\psi_t$ are $\frac{1}{\beta}$-strongly concave over domain $\nabla_g$ and $\nabla_f$ respectively. Also, the gradient of these functions is some $\cv \in [0,1]^d$, so that the norms of gradients are bounded by $\oneNorm$.  

Therefore, using online learning guarantees for smooth functions from Lemma \ref{lem:OCOstrongc}, along with $G = \oneNorm, H=1/\beta$, we get ${\cal R}(T) = O(\oneNorm^2 \beta \log T)$, and ${\cal R}'(T) = O(\oneNorm^2 \beta \log T)$. The theorem statement is obtained by substituting these \OCO~regret bounds in above.
\end{proof}
In above, observe that Assumption \ref{assum:grad} was required because Lemma \ref{lem:strongc} provided strong convexity of $g_t(\cdot)$ and $\psi_t(\cdot)$ only on the domains $\nabla_g$ and $\nabla_f$, respectively. We conjecture that it is possible to remove this assumption to get similar regret guarantees for the smooth case. 



\bibliographystyle{plainnat}
\bibliography{bibliography_secretary}

\begin{thebibliography}{42}
\providecommand{\natexlab}[1]{#1}
\providecommand{\url}[1]{\texttt{#1}}
\expandafter\ifx\csname urlstyle\endcsname\relax
  \providecommand{\doi}[1]{doi: #1}\else
  \providecommand{\doi}{doi: \begingroup \urlstyle{rm}\Url}\fi

\bibitem[Abernethy et~al.(2011)Abernethy, Bartlett, and Hazan]{blackwell2011}
Jacob Abernethy, Peter~L. Bartlett, and Elad Hazan.
\newblock Blackwell approachability and low-regret learning are equivalent.
\newblock In \emph{COLT}, 2011.

\bibitem[Aggarwal et~al.(2011)Aggarwal, Goel, Karande, and Mehta]{AGKM}
Gagan Aggarwal, Gagan Goel, Chinmay Karande, and Aranyak Mehta.
\newblock Online vertex-weighted bipartite matching and single-bid budgeted
  allocations.
\newblock In \emph{SODA}, 2011.

\bibitem[Agrawal et~al.(2014)Agrawal, Wang, and Ye]{AWY2009}
S.~Agrawal, Z.~Wang, and Y.~Ye.
\newblock A dynamic near-optimal algorithm for online linear programming.
\newblock \emph{Operations Research}, 62:\penalty0 876 -- 890, 2014.

\bibitem[Agrawal and Devanur(2014)]{AD14}
Shipra Agrawal and Nikhil~R. Devanur.
\newblock Bandits with concave rewards and convex knapsacks.
\newblock In \emph{Proceedings of the Fifteenth ACM Conference on Economics and
  Computation}, EC '14, 2014.

\bibitem[Arora et~al.(2012)Arora, Hazan, and Kale]{AHK12}
Sanjeev Arora, Elad Hazan, and Satyen Kale.
\newblock The multiplicative weights update method: a meta-algorithm and
  applications.
\newblock \emph{Theory of Computing}, 8\penalty0 (6):\penalty0 121--164, 2012.

\bibitem[Babaioff et~al.(2012)Babaioff, Dughmi, Kleinberg, and
  Slivkins]{babaioff2012}
Moshe Babaioff, Shaddin Dughmi, Robert Kleinberg, and Aleksandrs Slivkins.
\newblock Dynamic pricing with limited supply.
\newblock In \emph{EC}, 2012.

\bibitem[Badanidiyuru et~al.(2013)Badanidiyuru, Kleinberg, and Slivkins]{BwK}
Ashwinkumar Badanidiyuru, Robert Kleinberg, and Aleksandrs Slivkins.
\newblock Bandits with knapsacks.
\newblock In \emph{FOCS}, pages 207--216, 2013.

\bibitem[Bahmani and Kapralov(2010)]{BK10}
Bahman Bahmani and Michael Kapralov.
\newblock Improved bounds for online stochastic matching.
\newblock In \emph{ESA}, pages 170--181, 2010.

\bibitem[Blackwell(1956)]{blackwell1956}
David Blackwell.
\newblock An analog of the minimax theorem for vector payoffs.
\newblock \emph{Pacific Journal of Mathematics}, 6\penalty0 (1):\penalty0 1--8,
  1956.

\bibitem[Buchbinder et~al.(2007)Buchbinder, Jain, and Naor]{BJN}
Niv Buchbinder, Kamal Jain, and Joseph~Seffi Naor.
\newblock Online primal-dual algorithms for maximizing ad-auctions revenue.
\newblock In \emph{Proceedings of the 15th Annual European Conference on
  Algorithms}, ESA'07, 2007.

\bibitem[Chakrabarti and Vee(2012)]{VeeEC12b}
Deepayan Chakrabarti and Erik Vee.
\newblock Traffic shaping to optimize ad delivery.
\newblock In \emph{Proceedings of the 13th ACM Conference on Electronic
  Commerce}, EC '12, 2012.

\bibitem[Chen et~al.(2012)Chen, Ma, Mandalapu, Nagarjan, Shanmugasundaram,
  Vassilvitskii, Vee, Yu, and Zien]{VeeEC12a}
Peiji Chen, Wenjing Ma, Srinath Mandalapu, Chandrashekhar Nagarjan, Jayavel
  Shanmugasundaram, Sergei Vassilvitskii, Erik Vee, Manfai Yu, and Jason Zien.
\newblock Ad serving using a compact allocation plan.
\newblock In \emph{Proceedings of the 13th ACM Conference on Electronic
  Commerce}, EC '12, 2012.

\bibitem[Chen and Wang(2013)]{ChenWang2013}
Xiao Chen and Zizhuo Wang.
\newblock A near-optimal dynamic learning algorithm for online matching
  problems with concave returns.
\newblock {\tt http://arxiv.org/abs/1307.5934}, 2013.

\bibitem[Chen et~al.(2011)Chen, Berkhin, Anderson, and Devanur]{KDDRTB}
Ye~Chen, Pavel Berkhin, Bo~Anderson, and Nikhil~R. Devanur.
\newblock Real-time bidding algorithms for performance-based display ad
  allocation.
\newblock In \emph{Proceedings of the 17th ACM SIGKDD International Conference
  on Knowledge Discovery and Data Mining}, KDD '11, 2011.

\bibitem[Devanur and Hayes(2009)]{DH09}
Nikhil~R. Devanur and Thomas~P. Hayes.
\newblock The adwords problem: online keyword matching with budgeted bidders
  under random permutations.
\newblock In \emph{EC}, 2009.

\bibitem[Devanur and Jain(2012)]{DJ12}
Nikhil~R. Devanur and Kamal Jain.
\newblock Online matching with concave returns.
\newblock In \emph{Proceedings of the Forty-fourth Annual ACM Symposium on
  Theory of Computing}, STOC '12, 2012.

\bibitem[Devanur et~al.(2011{\natexlab{a}})Devanur, Jain, Sivan, and
  Wilkens]{DJSWfull}
Nikhil~R. Devanur, Kamal Jain, Balasubramanian Sivan, and Christopher~A.
  Wilkens.
\newblock Near optimal online algorithms and fast approximation algorithms for
  resource allocation problems.
\newblock Full version, accessible from
  \url{http://research.microsoft.com/en-us/um/people/bsivan/},
  2011{\natexlab{a}}.

\bibitem[Devanur et~al.(2011{\natexlab{b}})Devanur, Jain, Sivan, and
  Wilkens]{Devanur2011}
Nikhil~R. Devanur, Kamal Jain, Balasubramanian Sivan, and Christopher~A.
  Wilkens.
\newblock Near optimal online algorithms and fast approximation algorithms for
  resource allocation problems.
\newblock In \emph{EC}, 2011{\natexlab{b}}.

\bibitem[Devanur et~al.(2013)Devanur, Huang, Korula, Mirrokni, and
  Yan]{DHKMQ13}
Nikhil~R. Devanur, Zhiyi Huang, Nitish Korula, Vahab~S. Mirrokni, and Qiqi Yan.
\newblock Whole-page optimization and submodular welfare maximization with
  online bidders.
\newblock In \emph{Proceedings of the Fourteenth ACM Conference on Electronic
  Commerce}, EC '13, 2013.

\bibitem[Feldman et~al.(2009{\natexlab{a}})Feldman, Korula, Mirrokni,
  Muthukrishnan, and Pal]{FKMMP09}
J.~Feldman, N.~Korula, V.~Mirrokni, S.~Muthukrishnan, and M.~Pal.
\newblock Online ad assignment with free disposal.
\newblock In \emph{WINE}, 2009{\natexlab{a}}.

\bibitem[Feldman et~al.(2009{\natexlab{b}})Feldman, Mehta, Mirrokni, and
  Muthukrishnan]{FMMM}
Jon Feldman, Aranyak Mehta, Vahab Mirrokni, and S.~Muthukrishnan.
\newblock Online stochastic matching: Beating 1-1/e.
\newblock In \emph{FOCS '09: Proceedings of the 2009 50th Annual IEEE Symposium
  on Foundations of Computer Science}, 2009{\natexlab{b}}.

\bibitem[Feldman et~al.(2010{\natexlab{a}})Feldman, Henzinger, Korula,
  Mirrokni, and Stein]{Feldman10}
Jon Feldman, Monika Henzinger, Nitish Korula, Vahab~S. Mirrokni, and Cliff
  Stein.
\newblock Online stochastic packing applied to display ad allocation.
\newblock In \emph{Proceedings of the 18th Annual European Conference on
  Algorithms: Part I}, ESA'10, 2010{\natexlab{a}}.

\bibitem[Feldman et~al.(2010{\natexlab{b}})Feldman, Henzinger, Korula,
  Mirrokni, and Stein]{FHKMS10}
Jon Feldman, Monika Henzinger, Nitish Korula, Vahab~S. Mirrokni, and Clifford
  Stein.
\newblock Online stochastic ad allocation: Efficiency and fairness.
\newblock \emph{CoRR}, abs/1001.5076, 2010{\natexlab{b}}.

\bibitem[Ghosh et~al.(2009)Ghosh, McAfee, Papineni, and
  Vassilvitskii]{GhoshMPV09}
Arpita Ghosh, Randolph~Preston McAfee, Kishore Papineni, and Sergei
  Vassilvitskii.
\newblock Bidding for representative allocations for display advertising.
\newblock In \emph{WINE}, 2009.

\bibitem[Goel and Mehta(2008)]{GoelMehta}
Gagan Goel and Aranyak Mehta.
\newblock Online budgeted matching in random input models with applications to
  adwords.
\newblock In \emph{SODA '08: Proceedings of the nineteenth annual ACM-SIAM
  symposium on Discrete algorithms}, 2008.

\bibitem[Gupta and Molinaro(2014)]{GuptaM14}
Anupam Gupta and Marco Molinaro.
\newblock {How the Experts Algorithm Can Help Solve LPs Online}.
\newblock \emph{Algorithms - ESA 2014, Lecture Notes in Computer Science},
  8737:\penalty0 517--529, 2014.

\bibitem[Hazan et~al.(2007)Hazan, Agarwal, and Kale]{HazanLogarithmic}
Elad Hazan, Amit Agarwal, and Satyen Kale.
\newblock Logarithmic regret algorithms for online convex optimization.
\newblock \emph{Mach. Learn.}, 69\penalty0 (2-3), December 2007.

\bibitem[Hoeffding(1963)]{Hoeffding1963}
Wassily Hoeffding.
\newblock Probability inequalities for sums of bounded random variables.
\newblock \emph{Journal of the American Statistical Association}, 58\penalty0
  (301):\penalty0 13--30, March 1963.

\bibitem[Kakade et~al.(2009)Kakade, Shalev-Shwartz, and Tewari]{Kakade2009}
Sham~M. Kakade, Shai Shalev-Shwartz, and Ambuj Tewari.
\newblock {On the duality of strong convexity and strong smoothness: Learning
  applications and matrix regularization}.
\newblock Technical report, Toyota Technological Institute - Chicago, USA,
  2009.
\newblock
  \url{http://ttic.uchicago.edu/\~shai/papers/KakadeShalevTewari09.pdf}.

\bibitem[Karande et~al.(2011)Karande, Mehta, and Tripathi]{KMT11}
Chinmay Karande, Aranyak Mehta, and Pushkar Tripathi.
\newblock Online bipartite matching with unknown distributions.
\newblock In \emph{STOC}, 2011.

\bibitem[Karande et~al.(2013)Karande, Mehta, and Srikant]{KMS13}
Chinmay Karande, Aranyak Mehta, and Ramakrishnan Srikant.
\newblock Optimizing budget constrained spend in search advertising.
\newblock In \emph{Proceedings of the Sixth ACM International Conference on Web
  Search and Data Mining}, WSDM '13, 2013.

\bibitem[Karp et~al.(1990)Karp, Vazirani, and Vazirani]{KVV}
R.~M. Karp, U.~V. Vazirani, and V.~V. Vazirani.
\newblock An optimal algorithm for on-line bipartite matching.
\newblock In \emph{Proceedings of the Twenty-second Annual ACM Symposium on
  Theory of Computing}, STOC '90, 1990.

\bibitem[Kesselheim et~al.(2014)Kesselheim, T{\"o}nnis, Radke, and
  V{\"o}cking]{KTRV14}
Thomas Kesselheim, Andreas T{\"o}nnis, Klaus Radke, and Berthold V{\"o}cking.
\newblock {Primal beats dual on online packing LPs in the random-order model}.
\newblock In \emph{STOC}, 2014.

\bibitem[Kleinberg(2005)]{kleinberg05}
R.~Kleinberg.
\newblock {A multiple-choice secretary algorithm with applications to online
  auctions}.
\newblock In \emph{Proceedings of the 16th Annual ACM-SIAM Symposium on
  Discrete algorithms}, pages 630--631, January 2005.

\bibitem[Kleinberg et~al.(2008)Kleinberg, Slivkins, and Upfal]{Kleinberg2008}
Robert Kleinberg, Alex Slivkins, and Eli Upfal.
\newblock Multi-armed bandits in metric spaces.
\newblock In \emph{STOC}, 2008.

\bibitem[Mahdian and Yan(2011)]{MY11}
Mohammad Mahdian and Qiqi Yan.
\newblock {Online bipartite matching with random arrivals: an approach based on
  strongly factor-revealing LPs}.
\newblock In \emph{STOC}, 2011.

\bibitem[Mahdian et~al.(2012)Mahdian, Nazerzadeh, and Saberi]{MNS12}
Mohammad Mahdian, Hamid Nazerzadeh, and Amin Saberi.
\newblock Online optimization with uncertain information.
\newblock \emph{ACM Trans. Algorithms}, 8\penalty0 (1), January 2012.

\bibitem[Manshadi et~al.(2011)Manshadi, Gharan, and Saberi]{SS10}
Vahideh Manshadi, Shayan Gharan, and Amin Saberi.
\newblock Online stochastic matching: Online actions based on offline
  statistics.
\newblock In \emph{SODA}, 2011.

\bibitem[Mehta et~al.(2007)Mehta, Saberi, Vazirani, and Vazirani]{MSVV}
Aranyak Mehta, Amin Saberi, Umesh~V. Vazirani, and Vijay~V. Vazirani.
\newblock Adwords and generalized online matching.
\newblock \emph{J. ACM}, 54\penalty0 (5), 2007.

\bibitem[Mirrokni et~al.(2012)Mirrokni, Gharan, and Zadimoghaddam]{MGZ12}
Vahab~S. Mirrokni, Shayan~Oveis Gharan, and Morteza Zadimoghaddam.
\newblock Simultaneous approximations for adversarial and stochastic online
  budgeted allocation.
\newblock In \emph{Proceedings of the Twenty-third Annual ACM-SIAM Symposium on
  Discrete Algorithms}, SODA '12, 2012.

\bibitem[Shalev-Shwartz(2012)]{Shalev-Shwartz12}
Shai Shalev-Shwartz.
\newblock Online learning and online convex optimization.
\newblock \emph{Foundations and Trends in Machine Learning}, 4\penalty0
  (2):\penalty0 107--194, 2012.

\bibitem[Vee et~al.(2010)Vee, Vassilvitskii, and Shanmugasundaram]{VVS10}
Erik Vee, Sergei Vassilvitskii, and Jayavel Shanmugasundaram.
\newblock Optimal online assignment with forecasts.
\newblock In \emph{EC '10: Proceedings of the 11th ACM conference on Electronic
  commerce}, 2010.

\end{thebibliography}

\appendix

\renewcommand{\dist}{{}}
\newcommand{\conv}{\text{Conv}}
\section{Concentration Inequalities}\label{app:conc}
\begin{lemma}{\cite{Hoeffding1963}, Theorem 4}
\label{lem:HoeffdingReduction}
Let ${\cal X}=(x_1, \ldots, x_N)$ be a finite population of $N$ real points, $X_1, \ldots, X_n$ denote a random sample without replacement from ${\cal X}$, and $Y_1, \ldots, Y_n$ denote a random sample with replacement from ${\cal X}$. If $\ell: {\mathbb R} \rightarrow {\mathbb R}$ is continuous and convex, then
$$\Ex[\ell(\sum_{t=1}^n X_t)] \le \Ex[\ell(\sum_{t=1}^n Y_t)].$$
\end{lemma}

\begin{lemma}{\cite{Hoeffding1963}}
\label{lem:HoeffdingSwr}
 Let ${\cal X}=(x_1, \ldots, x_N)$ be a finite population of $N$ real points, $X_1, \ldots, X_n$ denote a random sample without replacement from ${\cal X}$.  Let $a=\min_{1\le i \le N} x_i$, $b=\max_{1\le i\le N} x_i$ and  $\mu= \frac{1}{N} \sum_{i=1}^N X_i$. Then, for all $\epsilon > 0$, 
$$ \Pr\left(\frac{1}{n} \sum_{i=1}^n X_i - \mu \ge \epsilon\right) \le \exp\left( -\frac{2n\epsilon^2}{(b-a)^2}\right). $$
\end{lemma}

\begin{lemma} (Multiplicative version) \label{lem:multiplicativeChernoff}
 Let ${\cal X}=(x_1, \ldots, x_N)$ be a finite population of $N$ real points, and $X_1, \ldots, X_n$ denote a random sample without replacement from ${\cal X}$.  Let $a=\min_{1\le i \le N} x_i$, $b=\max_{1\le i\le N} x_i$ and   $\mu= \frac{n}{N} \sum_{i=1}^N X_i$. Then, for all $\epsilon > 0 $, 
 $$ \Pr\left( | \sum_{i=1}^n X_i - \mu | \ge \epsilon \mu \right) \le \exp\left( -\frac{\mu\epsilon^2}{3(b-a)^2}\right). $$
\end{lemma}

\begin{corollary} (to \prettyref{lem:multiplicativeChernoff}) \label{cor:multiplicativeChernoff} 
	Let ${\cal X}=(x_1, \ldots, x_N)$ be a finite population of $N$ real points, and $X_1, \ldots, X_n$ denote a random sample without replacement from ${\cal X}$.  Let $a=\min_{1\le i \le N} x_i$, $b=\max_{1\le i\le N} x_i$ and   $\mu= \frac{n}{N} \sum_{i=1}^N X_i$. Then, for all $\rho > 0$, with probability at least $1-\rho$, 
	$$ |  \sum_{i=1}^n X_i - \mu | \le (b-a)\sqrt{3\mu \log(1/\rho)} $$
\end{corollary}

\begin{proof}
	Given $\rho > 0$, use \prettyref{lem:multiplicativeChernoff} with
	 \[ \epsilon =  {(b-a)} \sqrt{ \frac {3\log(1/\rho)} \mu } , \]
	 to get that the probability of the event $ |  \sum_{i=1}^n X_i - \mu | > \epsilon\mu = (b-a)\sqrt{3\mu \log(1/\rho)} $ is at most 
	 \[ \exp\left( -\frac{\mu\epsilon^2}{3(b-a)^2}\right) =  \exp\left( -\log(1/\rho)\right) = \rho.\]
\end{proof}

\begin{lemma}{\cite{Kleinberg2008, babaioff2012, BwK}}
\label{lem:concentration}
Consider a probability distribution with values in $[0,1]$, and expectation $\nu$. Let $\hat{\nu}$ be the average of $N$ independent samples from this distribution. 
Then, with probability at least $1-e^{-\Omega(\crad)}$, for all $\crad>0$, 
\begin{equation}
|\hat{\nu}-\nu| \le \rad(\hat{\nu}, N) \le 3\rad(\nu, N),
\end{equation}
where $\rad(\nu, N)=\sqrt{\frac{\crad \nu}{N}} + \frac{\crad}{N}.$
More generally this result holds if $X_1, \ldots, X_N \in [0,1]$ are random variables, $N\hat{\nu}=\sum_{t=1}^N X_t$, and $N\nu=\sum_{t=1}^N \Ex[X_t | X_1,\ldots, X_{t-1}]$.
\end{lemma}
\section{Preliminaries}
\label{app:prelims}
\subsection{Strong smoothness/Strong convexity duality.\newline}
\label{app:prelims:strongc}
{\noindent \it Proof of Lemma \ref{lem:strongc}}
Given $h$ is convex and $\beta$-strong smooth with respect to norm $||\cdot||$. We prove that $h^*$, defined as 
$$ h^*(\thetaV)=\max_{\y \in [0,1]^d} \{ \y \cdot \thetaV - h(\y)\},$$
is $\frac{1}{\beta}$-strongly convex with respect to norm $||\cdot||_*$ on domain $\nabla_h = \{\nabla h(\x): \x\in[0,1]^d\}$. 

For any $\thetaV, \phiV \in \nabla_h$, $\thetaV=\nabla h(\z), \phiV=\nabla h(\x)$ for some $\z,\x \in [0,1]^d$. And, therefore,
\begin{eqnarray}
\label{eq:scduality}
& & h^*(\thetaV) - h^*(\phiV) -\x \cdot (\thetaV-\phiV)\nonumber\\
& = &  h^*(\nabla h(\z)) - h^*(\nabla h(\x)) - \x\cdot (\nabla h(\z)-\nabla h(\x))\nonumber\\
& = &  \z \cdot \nabla h(\z) - h(\z) \nonumber\\
& & \ - (\x\cdot\nabla h(\x) - h(\x)) - \x \cdot (\nabla h(\z)-\nabla h(\x))\nonumber\\
& = &  \z \cdot \nabla h(\z) - h(\z) + h(\x) - \x \cdot \nabla h(\z)\nonumber\\
& = &  (\z-\x) \cdot (\nabla h(\z)-\nabla h(\x)) \nonumber\\
& & \ - (h(\z) - h(\x) -\nabla h(\x) (\z-\x))\nonumber\\
& = &  (\z-\x) \cdot (\nabla h(\z)-\nabla h(\x)) - g(\z-\x),\nonumber\\
\end{eqnarray}
where we define
$$g(\y) :=h(\x+\y)-h(\x)-(\nabla h(\x)) \cdot \y.$$ 
Now, for any ${\bf \varphi}$, 
\begin{eqnarray*}
g^*({\bf \varphi}) & : = & \sup_{\y} {\bf \varphi}\cdot \y - g(\y)  \\
& = & {\bf \varphi} \cdot \y^* - g(\y^*)
\end{eqnarray*}
where $\y^*$ is such that ${\bf \varphi}=\nabla g(\y^*) = \nabla h(\x + \y^*)-\nabla h(\x)$.
Therefore, for ${\bf \varphi}=\nabla h(\z)-\nabla h(\x)$, $\y^*=\z-\x$, so that, 
\begin{eqnarray*}
g^*(\nabla h(\z)-\nabla h(\x)) & = & (\nabla h(\z)-\nabla h(\x))\cdot (\z-\x) \\
& & \ \ - g(\z-\x).
\end{eqnarray*}
Substituting in \eqref{eq:scduality}, we get
\begin{eqnarray*}
& & h^*(\thetaV) - h^*(\phiV) -\x \cdot (\thetaV-\phiV)\nonumber\\
& = & g^*(\nabla h(\z)-\nabla h(\x)) \text{\commentNikhil{ Why is this?} }\\
& = & g^*(\thetaV-\phiV)
\end{eqnarray*}
By smoothness assumption, $g(\y)\le \frac{\beta}{2} ||\y||^2$. This implies that $g^*(\thetaV)\ge \frac{1}{2\beta}||\thetaV||_*^2 $ because 
the conjugate of $\beta$ times half squared norm is $1/\beta$ times half squared of the dual norm. 
This gives 
$$ h^*(\thetaV) - h^*(\phiV) -\x \cdot (\thetaV-\phiV) \ge \frac{1}{2\beta} \|\thetaV-\phiV\|_*^2.$$
This completes the proof.

\subsection{Online learning.}
\label{app:oco}
A popular algorithm for \OCO~is the online mirror descent (\OMD) algorithm. The \OMD~algorithm with regularizer $R(\thetaV)$ uses the following fast update rule to select player's decision $\thetaV_{t+1}$ for this problem:
\begin{eqnarray}
\label{eq:OMDupdate}
\thetaV_{t+1} & = & \arg \max_{\thetaV\in W} \frac{1}{\eta}R(\thetaV)-\thetaV \cdot \y_{t+1}, \text{ where} \nonumber\\
\y_{t+1} & = & \y_{t}-z_t, \text{ and } z_t\in \partial g_t(\thetaV_t)
\end{eqnarray}
The maximization problem in above is particularly simple when domain $W$ is of form $||\thetaV||\le \gamma$, and this is the main use case of this algorithm in this paper. 
Further, for domain $W$ of form $||\thetaV||_2\le L$, and $R(\thetaV)=||\thetaV||_2^2$, this simply becomes online gradient descent.
\OMD~has the following guarantees for this problem:
\begin{lemma}{\cite{Shalev-Shwartz12}}
 $$\regOCO(T) \le \frac{D}{\eta} + \eta TG^2,$$
where $D=(\max_{\thetaV''} R(\thetaV'') -\min_{\thetaV' \in W} R(\thetaV'))$, $\frac{1}{T}\sum_{t=1}^T ||z_t||^2 \le G$ for $z_t\in \partial g_t(\thetaV_t)$, and $R$ is a $1$-strongly-convex function with respect to norm $||\cdot||_*$. 
\end{lemma}
Now, to derive Corollary \ref{cor:regOCO}, observe that for $W=\{||\thetaV||_2\le L\}$, Euclidean regularizer $R(\thetaV)=||\thetaV||_2^2$ gives $\regOCO(T) \le LG\sqrt{T}$, with $G^2 = d\ge \frac{1}{T}\sum_{t=1}^T ||z_t||_2^2$, when $z_t\in [0,1]^d$. 
And, for $W=\{||\thetaV||_1\le L, \thetaV> 0\}$, entropic regularizer $R(\thetaV)=\sum_i \thetaV_i \log \thetaV_i$ gives $\regOCO(T) \le G\sqrt{LT \log(d)}$, where $G^2 =1 \ge \frac{1}{T}\sum_{t=1}^T ||z_t||_{\infty}^2$, when $z_t\in [0,1]^d$.


\section{Sampling without replacement bounds for Section \ref{sec:onlyS}}
\label{app:onlyS}
{\noindent \bf Proof of Equation \eqref{eq:swr}.} Let ${\boldsymbol \omega}=\Ex[{\bf w}_{t,\sigma}] = \Ex[\cv_t^*]$. To bound the quantity $\Ex[||{\bf w}_{t,\pi} - {\boldsymbol \omega}||]$, note that ${\bf w}_{t,\pi}$ can be viewed as the average of $t$ vectors sampled uniformly 
{\em without replacement} from the ground set $\{\cv_{X_{1}}, \ldots, \cv_{X_{T}}\}$ of $T$ vectors. 

Now, let ${w}_{t,\pi,j}$ denote the $j^{th}$ component of vector ${\bf {w}}_{t,\pi}$.
Then, by applying concentration bounds from Corollary \ref{cor:multiplicativeChernoff}, we get that 
$$|{w}_{t,\pi,j} - \omega_j| \le \sqrt{\frac{  3\omega_j\log(d/\rho)}{t}},$$
with probability $1-\frac{\rho}{d}$ for all $\rho\in (0,1)$. From the condition ${\boldsymbol \omega}=\Ex[\cv^*_{t}]\in S$, we have $\omega_j \le \max_{\cv\in S} \cvS_j \le s$. 
Taking union bound over $d$, for every $\rho \in (0,1)$, we have that with probability $1-\rho$,
$$\|{\bf w}_{t,\pi} -  {\boldsymbol \omega}\| \le \oneNorm \sqrt{\frac{ 3s\log(d/\rho)}{t}}.$$
And, integrating over $\rho$, we obtain,
$$\Ex[\|{\bf w}_{t,\pi} -   {\boldsymbol \omega}\|] \le O(\oneNorm \sqrt{\frac{ s\log(d)}{t}}).$$

{\noindent \bf High Probability bounds.} 
For high probability bounds, firstly from Equation \eqref{eq:gt} and \eqref{eq:highProbW},
\begin{eqnarray*}
\sum_t \Ex[g_t(\thetaV_t) | {\cal F}_{t-1}]  & \le  & \sum_t \|\Ex[\cv^*_{t} | {\cal F}_{t-1}] - \Ex[\cv^*_{t}]\| \\
& = & \sum_t \|{\bf w}_{t,\pi} - {\boldsymbol \omega}\|,
\end{eqnarray*}
for uniform at random orderings $\pi$.

Then, as in above, using Corollary \ref{cor:multiplicativeChernoff} we obtain that for every $t$, with probability $1-\frac{\rho}{T}$
$$||{\bf w}_{t,\pi} -  {\boldsymbol \omega}|| \le \oneNorm \sqrt{\frac{ 3s\log({dT}/{\rho})}{t}}.$$
Taking union bound over $t=1,\ldots, T$, and summing over $t$ we obtain that with probability $1-\rho$,
\begin{eqnarray*}
\sum_t \Ex[g_t(\thetaV_t) | {\cal F}_{t-1}] & \le & \sum_t \|{\bf w}_{t,\pi} - {\boldsymbol \omega}\| \\
& = & O(\|{\bf 1}_d\| \sqrt{T\log(dT/\rho)}).
\end{eqnarray*}

Now, using Lemma \ref{lem:concentration} for dependent random variables $X_t=g_t(\thetaV_t)$, with $|X_t|=|\thetaV_t\cdot \vplay - h_S(\thetaV_t)| \le \|{\bf 1}_d\|$, we have, 
$$\sum_t g_t(\thetaV_t) - \sum_t \Ex[g_t(\thetaV_t) | {\cal F}_{t-1}]  \le O(\|{\bf 1}_d\| \sqrt{T\log(1/\rho)})$$
with probability at least $1-\rho$.

Combining the above observations, we obtain that with probability $1-\rho$,
$$ \sum_t g_t(\thetaV_t)  \le O(\|{\bf 1}_d\| \sqrt{T\log(dT/\rho)}).$$

\newcommand{\muV}{{\bs{\mu}}}
\section{Proof of Lemma \ref{lem:Z}}
\label{app:lem:Z}
The offline optimal solution needs to pick $\vopt\in \conv(X_t)$ to serve request type $X_t$, where $\conv(X_t)$ denotes the convex hull of set $X_t$. Therefore, $\OPT_\dist^\delta$ is defined as

\begin{eqnarray}
\OPT_\dist^\delta 
& := & \begin{array}{rcl} 
\max_{\{\cv_t\in \conv(X_t)\}}  & f (\frac{1}{T} \sum_t \cv_t) & \\
&  d(\frac{1}{T} \sum_t \cv_t,S) \le \delta &\\
\end{array}\nonumber\\
& = & \min_{\lambda \ge 0} \max_{\{\x=\frac{1}{T} \sum_t \cv_t, \cv_t\in \conv(X_t)\}} \left\{\right.\nonumber\\
& & \ \ \ f (\x) - \lambda d(\x,S) + \delta \lambda \left.\right\}\nonumber\\
& = & \min_{\lambda \ge 0} \max_{\{\x=\frac{1}{T} \sum_t \cv_t, \cv_t\in \conv(X_t)\}} \min_{||\phiV||_*\le L, ||\thetaV||_*\le 1} \left\{\right.\nonumber\\
& & \ \ \ f^*(\phiV) - \phiV\cdot \x - \lambda \thetaV\cdot \x + \lambda h_S(\thetaV) + \delta \lambda \left.\right\}\nonumber\\
& = & \min_{\lambda \ge 0, ||\phiV||_*\le L, ||\thetaV||_*\le 1} \max_{\{\x=\frac{1}{T} \sum_t \cv_t, \cv_t\in \conv(X_t)\}} \left\{\right.\nonumber\\
& & \ \ \ f^*(\phiV) - \phiV\cdot \x - \lambda \thetaV\cdot \x + \lambda h_S(\thetaV) + \delta \lambda \left.\right\}\nonumber\\
& = & \min_{\lambda \ge 0, ||\phiV||_*\le L, ||\thetaV||_*\le 1} \left\{\right. f^*(\phiV) + \lambda h_S(\thetaV) \nonumber\\
& & \ \ \ + \frac{1}{T} \sum_{t=1}^T h_{\conv(X_t)}(-\phiV-\lambda \thetaV) + \delta \lambda  \left.\right\}\nonumber\\
\end{eqnarray}
where,  recall that for any convex set $X$, $h_X(\thetaV)$ was defined as $h_X(\thetaV):= \max_{\cv\in X} \thetaV\cdot \cv$. Because a linear function is maximized at a vertex of a convex set, $h_{\conv(X_t)}(-\phiV-\lambda \thetaV)$ is same as $h_{X_t}(-\phiV-\lambda \thetaV)$. This allows us to rewrite the expression for $\OPT^\delta$ as 
\begin{eqnarray}
\label{eq:OPTexpr}
 \OPT^\delta & = & \min_{\lambda \ge 0, ||\phiV||_*\le L, ||\thetaV||_*\le 1} \left\{\right. f^*(\phiV) + \lambda h_S(\thetaV) \nonumber\\
& & \ + \frac{1}{T} \sum_{t=1}^T h_{X_t}(-\phiV-\lambda \thetaV) + \delta \lambda \left.\right\}
\end{eqnarray}

From above, it is clear that $\OPT^\delta$ is a non-decreasing concave function of $\delta$, with gradient as $\lambda^*(\delta) \ge 0$, where $\lambda^*(\delta)$ is the optimal dual variable corresponding to the distance constraint. And, 
$$ \lim_{\delta \rightarrow 0} \frac{\OPT^\delta -\OPT}{\delta} = \lambda^*$$
where $\lambda^*$ is the optimal dual variable for $\OPT$ (i.e., the case of $\delta=0$).
This proves the lemma.

\section{Proof of Theorem \ref{th:cp}}
\label{app:CP}
We provide proof of a more detailed theorem statement.
\begin{theorem}
Given $Z$ that satisfies Assumption \ref{assum:Z}, Algorithm \ref{algo:cp} achieves the following regret bounds for online stochastic \CP, in \RP~model:
\begin{center}
$ \Ex[\areg_1(T)] \le \frac{(Z+L)}{T}\cdot O\left(\regOCO(T) + {\cal Q}(T)\right) + O(\frac{\regOCO'(T)}{T}),$
\end{center}
\begin{center}
$ \Ex[\areg_2(T)] \le \frac{1}{T} \cdot O\left(\regOCO(T) + {\cal Q}(T)\right) + \frac{1}{(Z+L)}O(\frac{\regOCO'(T)}{T}),$
\end{center}
where ${\cal Q}(T)= O(||{\bf 1}_d|| \sqrt{sT\log(d)})$, 
$\regOCO'(T)$ is the regret bound for \OCO~on $\psi_t(\cdot)$, $\regOCO(T)$ is the regret bound for \OCO~on $g_t(\cdot)$. 
And, $s\le 1$ is the coordinate-wise largest value a vector in $S$ can take.
\end{theorem} 
Then, substituting \OCO~regret bounds from Corollary \ref{cor:regOCO} gives the statement of Theorem \ref{th:cp}.
\begin{proof}
Denote by $(\vopt)$ the choice made by the offline optimal solution to satisfy request $A_t$. Then,
$$ f(\Ex[\vopt]) \ge \OPT,\text{ and } \Ex[\vopt] \in S,$$
where expectation is over $A_t$ drawn uniformly at random  from $X_1, \ldots, X_T$.
 
Lemma \ref{lem:gtRPgeneral} provides
\begin{eqnarray*}
f(\Ex[\vopt])+\frac{1}{T}\sum_t \Ex[\psi_t(\phiV_t) + 2(Z+L) g_t(\thetaV_t)] &&\\
\le (Z+L)\frac{{\cal Q}(T)}{T} &&
\end{eqnarray*}

where ${\cal Q}(T)=O(||{\bf 1}_d|| \sqrt{s\log(d)T})$. 
Using Fenchel duality and \OCO~guarantees, it follows that
\begin{eqnarray*}
\min_{||\thetaV||_*\le 1} \frac{1}{T} \sum_t g_t(\thetaV) & = & d(\frac{1}{T}\sum_t \vplay, S) \\
& \le & \frac{1}{T} \sum_t g_t(\thetaV_t) + \frac{1}{T}\regOCO(T), 
\end{eqnarray*}
$$\min_{||\phiV||_*\le L} \psi_t(\phiV) = -f(\frac{1}{T}\sum_t \vplay) \le \frac{1}{T} \sum_t \psi_t(\thetaV_t) + \frac{1}{T}\regOCO'(T).$$
Then, using above observations, along with $f(\Ex[\vopt]) \ge \OPT$, we obtain
\begin{eqnarray*}
 & & \OPT -\Ex[f(\frac{1}{T}\sum_t \vplay)] + 2(Z+L) \Ex[d(\frac{1}{T}\sum_t \vplay, S)] \\
 & & \le \frac{2(Z+L)}{T}({\cal Q}(T) + \regOCO(T)) - \frac{1}{T}\regOCO'(T).
\end{eqnarray*}
This gives
\begin{eqnarray}
\label{eq:key2}
\Ex[f(\frac{1}{T}\sum_t \vplay)]  & \ge & \OPT +2(Z+L) \Ex[d(\frac{1}{T}\sum_t \vplay, S)] \nonumber\\
& & - \frac{2(Z+L)}{T}({\cal Q}(T) + \regOCO(T)) \nonumber \\
& & - \frac{1}{T}\regOCO'(T)
\end{eqnarray}
Now, we use Assumption \ref{assum:Z}, 
to upper bound the reward obtained by the algorithm in terms of \OPT~and distance from set $S$. In particular, we obtain that for $\delta:=\Ex[d(\frac{1}{T}\sum_t \vplay, S)]$,
\begin{eqnarray}
\label{eq:Zapplication2}
\Ex[ f(\frac{1}{T}\sum_t\vplay)] & \le  & f(\Ex[\frac{1}{T}\sum_t\vplay]) \nonumber\\
& \le & \OPT^\delta_\dist \nonumber\\
& \le & \OPT_\dist + Z \delta \nonumber\\
& = & \OPT+Z \cdot \Ex[d(\frac{1}{T}\sum_t \vplay, S)].
\end{eqnarray}
Combining the above two inequalities, we obtain
$$ \Ex[d(\frac{1}{T}\sum_t \vplay, S)] \le \frac{2}{T} (\regOCO(T) +{\cal Q}(T)) + \frac{1}{(Z+L)} \regOCO'(T).$$
And, from \eqref{eq:key2} (using $\Ex[d(\frac{1}{T}\sum_t \vplay, S)]\ge 0$),
\begin{eqnarray}
& & \hspace{-0.3in}\Ex[ f(\frac{1}{T}\sum_t\vplay)]\nonumber\\
& \ge & \OPT-\frac{2(Z+L)}{T}\cdot\left(\regOCO(T) +{\cal Q}(T)\right) - \frac{\regOCO'(T)}{T}.\nonumber\\
\end{eqnarray}
This gives the theorem statement.

\comment{For \IID~model, denote by $(\ropt, \vopt)$  the option chosen to satisfy request $A_t$ by the {\em Pure-random} algorithm. Recall that that the {\em Pure-random algorithm} is a non-adaptive algorithm that uses the optimal solution of the expected instance for distribution $\dist$ to satisy request $A_t$.  By this definition, we have that 
$$ \Ex_{A_t \sim \dist}[\ropt] \ge \OPT_{\dist},\text{ and } \Ex_{A_t \sim \dist}[\vopt] \in S$$
Then, the above arguments can be repeated, using Lemma \ref{lem:gtIIDreward} to get the same bounds with ${\cal Q}(T)=0$. In particular, we obtain 
$$\Ex[d(\frac{1}{T}\sum_t \vplay, S)] \le \frac{2}{T} \regOCO(T)+ \frac{1}{Z}\regOCO'(T), \text{ and,}$$
$$\Ex[ f(\frac{1}{T}\sum_t\vplay)] \ge \OPT_\dist-\frac{2Z}{T}(\regOCO(T)+\regOCO'(T)) \ge \OPT-\frac{2Z}{T}\regOCO(T)-\frac{1}{T}\regOCO'(T).$$
}

\end{proof}
\comment{
\begin{lemma}
\label{lem:gtIIDgeneral}
For \IID~and \ASI~model , 
$$f(\Ex[\vopt])+\frac{1}{T}\sum_t \Ex[\psi_t(\phiV_t) + Z g_t(\thetaV_t)]\le 0.$$
\end{lemma}
\begin{proof}
\begin{eqnarray*}
\psi_t(\phiV_t) + 2Z g_t(\thetaV_t)  & = & +\phiV \cdot\vplay -f^*(\phiV)+ 2Z(\thetaV_t\cdot \vplay - h_{S}(\thetaV_t)) \\
& \le & \phiV \cdot\vopt -f^*(\phiV)+ 2Z(\thetaV_t\cdot \vopt - h_{S}(\thetaV_t)) .
\end{eqnarray*}
\begin{eqnarray}
\Ex[\psi_t(\phiV_t) + 2Z g_t(\thetaV_t) | {\cal F}_{t-1}] & \le  &  -f^*(\phiV_t)+\phiV_t \cdot \Ex[\vopt | {\cal F}_{t-1}] + 2Z(\thetaV_t\cdot \Ex[\vopt | {\cal F}_{t-1}] - h_{S}(\thetaV_t)) \nonumber \\
 & =  &  -f^*(\phiV_t)+\phiV_t \cdot \Ex[\vopt] + 2Z(\thetaV_t\cdot \Ex[\vopt] - h_{S}(\thetaV_t)) \nonumber \\
& \le &  -f(\Ex[\vopt]) + 2Z(d(\Ex[\vopt], S)) \nonumber\\
& = &  -f(\Ex[\vopt])
\end{eqnarray}
\end{proof}
}
\begin{lemma}
\label{lem:gtRPgeneral}
\begin{eqnarray*}
& &  f(\Ex[\vopt])+\frac{1}{T}\sum_t \Ex[\psi_t(\phiV_t) + 2(Z+L) g_t(\thetaV_t)] \\
& & \ \ \ \le \frac{1}{T} (Z+L) O(||{\bf 1}_d||\sqrt{sT\log(d)}).
\end{eqnarray*}
\end{lemma}
\begin{proof}
\begin{eqnarray*}
& &  \hspace{-0.3in} \psi_t(\phiV_t) + 2(Z+L) g_t(\thetaV_t)  \\
& = & \phiV \cdot\vplay -(-f)^*(\phiV)+ 2(Z+L)(\thetaV_t\cdot \vplay - h_{S}(\thetaV_t)) \\
& \le & \phiV \cdot\vopt -(-f)^*(\phiV) + 2(Z+L)(\thetaV_t\cdot \vopt - h_{S}(\thetaV_t)) .
\end{eqnarray*}
\begin{eqnarray*}
& &  \hspace{-0.3in} \Ex[\psi_t(\phiV_t) + 2(Z+L) g_t(\thetaV_t) | {\cal F}_{t-1}] \\
& \le  &  \phiV_t \cdot \Ex[\vopt | {\cal F}_{t-1}] -(-f)^*(\phiV_t) \\
& & \ \ + 2(Z+L)(\thetaV_t\cdot \Ex[\vopt | {\cal F}_{t-1}] - h_{S}(\thetaV_t)) \nonumber \\
& \le &  -f(\Ex[\vopt]) +\phiV_t \cdot (\Ex[\vopt | {\cal F}_{t-1}]-\Ex[\vopt]) \\
& & \ \ + 2(Z+L) \thetaV_t\cdot(\Ex[\vopt | {\cal F}_{t-1}]-\Ex[\vopt])
\end{eqnarray*}
where the last inequality uses $\phiV_t \cdot \Ex[\vopt]-(-f)^*(\phiV_t) \le -f(\Ex[\vopt])$ (using Fenchel duality) and $\thetaV_t\cdot \Ex[\vopt] - h_{S}(\thetaV_t) \le d(\Ex[\vopt], S)=0$.
Then, as in proof of Lemma \ref{lem:gtRP}, $\Ex[\sum_t ||\Ex[\vopt | {\cal F}_{t-1}]-\Ex[\vopt]||]$ can be upper bounded by $O(\sqrt{||{\bf 1}_d||sT\log(d)})$. Using this along with observation that $||\phiV_t||_*\le L, ||\thetaV_t||_*\le 1$, we get the desired lemma statement.

\end{proof}
\newcommand{\gammaDef}[1]{||{\bf 1}_d||\sqrt{\frac{\log(d#1)}{#1}}}
\newcommand{\gammaDefRho}[1]{||{\bf 1}_d||\sqrt{\frac{\log(d/\rho)}{#1}}}
\section{Estimating the parameter $Z$} 
\label{app:estZ}
Let $Z^*$ denote the minimum value of $Z$ that satisfies the property in Equation \eqref{eq:Zprop}. As discussed in the proof of Lemma \ref{lem:Z}, $Z^*=\lambda^*$, the value of optimal dual variable corresponding to feasibility constraint.
To obtain low regret bounds, ideally we would like to use $Z=Z^*$ in Algorithm \ref{algo:cp}, which would provide the minimum possible regret bound of $O((Z^*+L)\sqrt{\frac{C}{T}})$  in objective according to Theorem \ref{th:cp}. The regret in constraints does not depend on $Z$. However, in the absence of knowledge of $Z^*$, we need to obtain a good enough approximation. Following lemma provides a 
 relaxed condition to be satisfied by $Z$ in order to obtain the same order of regret bounds, as those obtained with $Z=Z^*$.

\begin{lemma}
\label{lem:relaxed2}
Assume that $Z \ge 0$ satisfies the following property,
for all $\delta \ge 3\gamma$ where $\gamma =||{\bf 1}_d||\sqrt{\frac{\log(dT)}{T}}$, 
$$ \frac{\OPT^{\delta}_\dist- \OPT^{2\gamma}_\dist}{\delta} \le Z = O(Z^*+L).$$
Then, Algorithm \ref{algo:cp} using such a $Z$ will achieve an expected regret bound of $O((Z^*+L)\gamma)$ in objective, and $O(\gamma)$ in constraints. 

To compare with Theorem \ref{th:cp}, note that $\gamma=O(\sqrt{\frac{C \log(T)}{T}})$, therefore, using such a $Z$ degrades the regret bounds by only an $O(\sqrt{\log(T)})$ factor.
\end{lemma}
\begin{proof}
Recall that in the proof of Theorem \ref{th:cp}, the condition $\OPT^\delta \le \OPT+Z\delta$ was used in the following way. We had the inequality,
\begin{eqnarray}
\label{eq:start}
\OPT^{\Ex[d(\vavg,S)]} & \ge & \Ex[f(\vavg)] \nonumber\\
& \ge & \OPT + 2(Z+L) \Ex[d(\vavg,S)] \nonumber\\
& & \ - \ell(T) ,
\end{eqnarray}
where $\ell(T)=O((Z+L) \sqrt{\frac{C}{T}})$.
Then, we applied $\OPT^{\Ex[d(\vavg,S)]}  \le \OPT+Z \Ex[d(\vavg,S)]$, to obtain $\OPT + Z \Ex[d(\vavg,S)] \ge \OPT + 2(Z+L) \Ex[d(\vavg,S)] - \ell(T)$, yielding $\Ex[d(\vavg, S)] \le \frac{1}{(Z+L)}O(\ell(T))=O(\sqrt{\frac{C}{T}})$. 

Now, we will show that it suffices to have $Z\ge \frac{\OPT^{\delta}_\dist- \OPT^{2\gamma}_\dist}{\delta} $, for $\delta>3\gamma$ to obtain the given regret bounds.

We first bound $\Ex[\areg_2(T)] =\Ex[d(\vavg,S)]$. Starting with Equation \ref{eq:start}, observe that if $\Ex[d(\vavg, S)]\le 3\gamma$, then the distance is bounded by $O(\gamma)$ as required anyway, therefore, assume that $\delta:=\Ex[d(\vavg, S)]\ge 3\gamma$. Then, from the given property of $Z$ we have $\OPT^{\Ex[d(\vavg,S)]} = \OPT^{\delta} \le \OPT^{2\gamma}+Z \delta = \OPT^{2\gamma}+Z \Ex[d(\vavg, S)]$. Substituting back in Equation \eqref{eq:start}, we get
 \begin{eqnarray*}
& & \OPT^{2\gamma} + Z \Ex[d(\vavg,S)]\\
 & \ge & \OPT + 2(Z+L) \Ex[d(\vavg,S)] - \ell(T)
\end{eqnarray*}
which gives
  \begin{eqnarray*}
	(Z+L) \Ex[d(\vavg,S)] & \le & \ell(T) + \OPT^{2\gamma} - \OPT \\
	& \le & \ell(T) + 2Z^*\gamma \\
	& = & O((Z+L)\gamma) + 2Z^*\gamma
	\end{eqnarray*}
Then, using $Z=O(Z^*+L)$,
we get 
 \begin{eqnarray*}
 \Ex[\areg_2(T)] & = &\Ex[d(\vavg,S)] \\
& = & O(\gamma) = O(\sqrt{\frac{C\log(T)}{T}}).
\end{eqnarray*}
The bound on $\Ex[\areg_1(T)]$ depends only on the upper bound on $Z$ used, and $Z = O(Z^*+L)$ makes this regret bound to be $O((Z^*+L)\sqrt{\frac{C}{T}})$.
\end{proof}
Next, we provide method for estimating a $Z$ that satisfies the property stated in Lemma \ref{lem:relaxed2}. 
Define
\begin{eqnarray}
\label{eq:OPTest}
\OPTest^\delta_\dist(n) & = & \begin{array}{rcl} 
\max_{\{\cv_t\in \conv(A_t)\}}  & f(\frac{1}{n} \sum_{t=1}^{n}  \cv_t) & \nonumber\\
& d(\frac{1}{n} \sum_{t=1}^n  \cv_t, S) \le \delta&\\
\end{array}\\
\end{eqnarray}
with $\OPTest(n)$ denoting $\OPTest^\delta(n)$ for $\delta=0$. 
We will divide the timeline into  phases of size $1,1, 2^1, 2^2,...., 2^r, \ldots$. Note that phase $r\ge 2$ consists of $T_r=2^{r-2}$ time steps, and there are $T_r$ time steps before phase $r$. The first phase of a single step, we make an arbitrary choice. Then, in every phase $r\ge 2$, we will rerun the algorithm, using $Z$ constructed using observations from the previous $T_r$ time steps as 
\begin{equation}
\label{estimateOFZ}
Z :=\frac{(\OPTest^{4\gamma}_\dist(T_r)-\OPTest^\gamma(T_r))}{\gamma} + 2L
\end{equation}
with $\gamma=\gammaDef{T_r}$.  

\begin{algorithm}{\scshape [{Algorithm for online CP \\with Z estimation}]}
\label{algo:cpZ}
  \begin{algorithmic}
	\STATE
Choose any option in the first step.
\FORALL{phases $r=2,..., \log(T)+1$}
			\STATE COMPUTE $Z$ using observations in steps $1$ to $T_r=2^{r-2}$ as
			$$ Z=\frac{(\OPTest^{4\gamma}_\dist(T_r)-\OPTest^\gamma(T_r))}{\gamma} + 2L$$
with $\gamma=\gammaDef{T_r}$. 
		\STATE Run Algorithm \ref{algo:cp} for $T_r $ steps $t=\{T_r+1, \ldots, 2T_r\}$ of phase $r$ using $Z$ as computed above.
\ENDFOR
\end{algorithmic}
\end{algorithm}


We prove the following lemma regarding the estimate $Z$ used in above. 
Here we use the observation that in \RP~model, the first $n$ time steps provide a random sample of observations from the $T$ observations.

\begin{lemma}
\label{lem:estimation}

For all $\rho > 0$ and for all natural numbers $n$,  let $\gamma=\gammaDefRho{n}$, 
and $$Z := \frac{(\OPTest^{4\gamma}_\dist(n)-\OPTest^\gamma(n))}{\gamma} + 2L.$$ 
Then, for all $\delta>3\gamma$, with probability $1-O(\rho)$, 
$$ \frac{(\OPT^{\delta}-\OPT^{2\gamma})}{\delta} \le Z \le O(L+ Z^*).$$
\end{lemma}
The proof of above lemma is provided later. 
We now state the regret bounds for Algorithm \ref{algo:cpZ}. 
\begin{theorem}
	Algorithm \ref{algo:cpZ} has an expected regret of $ \tilde{O}(\sqrt{\frac{C}{T}}) $ in the objective and $  (Z^*+L)\tilde{O}(\sqrt{\frac{C}{T}})$ in the constraints. 
\end{theorem} 
\begin{proof} 
	For phase $r\ge 2$, using $n=2^{r-2}= T_r$, the number of time steps in phase $r$, and $\rho=\frac{1}{T_r^2}$, from Lemma \ref{lem:estimation} we obtain that with probability $1-O(\frac{1}{T_r^2})$, $Z$ available to phase $r$ satisfies the property required by Lemma \ref{lem:relaxed2} (with $T$ substituted by $T_r$), which gives the following regret bounds for phase $r$:  let $\vplayavg(r)$ be the average of played vectors in the $T_r$ time steps of phase $r$. Let ${\cal F}_{r-1}$ denote the history till phase $r-1$. Then, with probability $1-O(\frac{1}{T_r^2})$ the history ${\cal F}_{r-1}$ is such that in phase $r$ the regret in distance is bounded by
$\Ex[d(\vplayavg(r), S) | {\cal F}_{r-1}] \le  \tilde{O}(\sqrt{\frac{C}{T_r}})$.
With remaining probability $O(\frac{1}{T_r^2})$, the distance can be at most $T_r ||{\bf 1}_d||$. Let $\vplayavg$ denote the average of played vectors from the entire period of $T$ time steps. Then, we get that total regret,
\begin{eqnarray*}
\Ex[d(\vplayavg, S)] & \le & \frac{\oneNorm}{T}+\sum_{r=2}^{\log(T)+1} \frac{T_r}{T} \Ex[d(\vplayavg(r), S)] \\
& \le & \frac{\oneNorm}{T}+\sum_{r=2}^{\log(T)+1} \frac{T_r}{T} \tilde{O}(\sqrt{\frac{C}{T_r}} + \frac{T_r||{\bf 1}_d||}{T_r^2})   \\
& = & \tilde{O}(\sqrt{\frac{C}{T}}).
\end{eqnarray*}
Similarly, we  obtain bounds on regret in the objective,
\begin{eqnarray*}
& & \hspace{-0.3in} \OPT-\Ex[f(\vplayavg)] \\
& \le & \frac{1}{T}+\sum_{r=2}^{\log(T)+1} \frac{T_r}{T} (\OPT-\Ex[f(\vplayavg(r))]) \\
& \le & \frac{1}{T}+\sum_{r=2}^{\log(T)+1} \frac{T_r}{T}(Z^*+L) \tilde{O}(\sqrt{\frac{C}{T_r}} + \frac{T_r ||{\bf 1}_d||}{T_r^2}) \\
& = & (Z^*+L)\tilde{O}(\sqrt{\frac{C}{T}}).
\end{eqnarray*}
\end{proof} 

\begin{proof}[Proof of Lemma \ref{lem:estimation}]
From Lemma \ref{lem:Z}, $\OPT^\delta$ is concave in $\delta$, therefore, for all $\delta >3\gamma$
\begin{eqnarray*}
\frac{(\OPT^{\delta}_\dist-\OPT^{2\gamma})}{\delta} & \le & \frac{(\OPT^{\delta}_\dist-\OPT^{2\gamma})}{\delta-2\gamma} \\
& \le & \frac{(\OPT^{3\gamma}_\dist-\OPT^{2\gamma})}{\gamma}.
\end{eqnarray*}
So, it suffices to prove that
\[ \frac{(\OPT^{3\gamma}_\dist-\OPT^{2\gamma})}{\gamma}  \le  Z  \le  O(L+ Z^*).\]
In Lemma \ref{lem:1} and Lemma \ref{lem:2}, we prove that for every $\delta\ge \gamma$, with probability $1-O(\rho)$
\begin{eqnarray}
\label{eq:1}
\OPTest^{\delta} + L\gamma & \ge & \OPT^{\delta-\gamma}, \nonumber\\
\OPT^{\delta} + L\gamma & \ge & \OPTest^{\delta-\gamma}
\end{eqnarray}
Using above for $\delta=4\gamma$, and $\delta=2\gamma$, respectively, we get 
\begin{eqnarray*}
Z & := & \frac{(\OPTest^{4\gamma}_\dist(n)-\OPTest^\gamma(n))}{\gamma} + 2L\\
& \ge & \frac{(\OPT^{3\gamma}_\dist-\OPT^{2\gamma})}{\gamma}.
\end{eqnarray*}
In Lemma \ref{lem:3}, we  prove that for any $\delta \ge \gamma$,
\begin{equation}
\label{eq:2}
\OPTest^\delta \le \OPT + O(\delta (Z^*+L))
\end{equation}
Using this along with $\OPTest^\gamma \ge \OPT - L \gamma$ from the first inequality in Equation \eqref{eq:1}, we get
\begin{eqnarray*}
Z & = & \frac{\OPTest^{4\gamma} - \OPTest^{\gamma}}{\gamma} \\
& \le & \frac{(\OPT+4\gamma O(Z^*+L)) - (\OPT-L\gamma)}{\gamma} \\
& = & O(Z^*+L).
\end{eqnarray*}
This completes the proof.
\end{proof}

\begin{lemma}
\label{lem:conc}
Given fixed $\{\cv_t\}_{t=1}^T$, and a vector $\muV$, for all $\rho>0$ and $n\in [T]$, let  $\gamma=\gammaDefRho{n}$. Then for a uniformly random permutation over $1,\ldots, T$, with probability $1-O(\rho)$, the following holds for the first $n$ time steps.
		$$ \|\frac{1}{n} \sum_{t=1}^n \cv_t -\frac{1}{T} \sum_{t=1}^T \cv_t\| \le \gamma,$$
		$$ \left|\frac{1}{n} \sum_{t=1}^n h_{A_t}(\muV) - \frac{1}{T} \sum_{t=1}^T h_{A_t}(\muV)\right| \le \gamma \|\muV\|_*.$$
\end{lemma}
\begin{proof}
The first inequality is obtained by simple application of Chernoff-Hoeffding bounds (Lemma \ref{lem:HoeffdingSwr}) for every coordinate $\cvS_{t,j}$, which gives
$$ \left|\frac{1}{n} \sum_{t=1}^n \cvS_{t,j} -\frac{1}{T} \sum_{t=1}^T \cvS_{t,j}\right| \le \sqrt{\frac{\log(d/\rho)}{n}},$$
with probability $1-O(\rho/d)$. Then taking union bound over the $d$ coordinates, we get the required inequality.
 
The second inequality follows using Chernoff-Hoeffding bounds (Lemma \ref{lem:HoeffdingSwr}) for bounded random variables $Y_t=h_{A_t}(\muV)$, where $|Y_t|=|h_{A_t}(\muV)| \le ||\muV||_*\cdot||{\bf 1}_d||$ (from the definition of the dual norm). This gives with probability $1-O(\rho)$, 
\begin{eqnarray*}
& & \hspace{-0.2in} \left| \frac{1}{n} \sum_{t=1}^n h_{A_t}(\muV) - \frac{1}{T} \sum_{t=1}^T h_{A_t}(\muV) \right| \\
& = & |\frac{1}{n} \sum_{t=1}^n (Y_t - \Ex[Y_t])| \\
& \le & (||\muV||_*\cdot||{\bf 1}_d||) \sqrt{\frac{\log(1/\rho)}{n}} \\
& \le & ||\muV||_* \gamma.
\end{eqnarray*}
\end{proof}

\begin{lemma}
\label{lem:1}
For all $\rho>0 $ and $n\in [T]$, let  $\gamma=\gammaDefRho{n}$. For all $\delta \ge \gamma$, with probability $1-O(\rho)$, 
$$\OPTest^{\delta}(n) \ge \OPT^{\delta-\gamma} - L\gamma.$$
\end{lemma}
\begin{proof}
To prove $\OPTest^{\delta}(n) \ge \OPT^{\delta-\gamma} -L\gamma$, we prove that there exists a feasible primal solution of $\OPTest^\delta(n)$ that is at most $\gamma$ distance from the optimal primal solution of $\OPT^{\delta-\gamma}$ . Then, the lemma follows from the $L$-Lipschitz property of $f$.

Let $\{\cv_t\}_{t=1}^T$ be the optimal primal solution for $\OPT^{\delta-\gamma}$, so that $d(\frac{1}{T} \sum_{t=1}^T \cv_t, S) \le \delta-\gamma.$
Then,
\begin{eqnarray*}
& & \hspace{-0.3in} d(\frac{1}{n} \sum_{t=1}^n \cv_t , S) \\
& \le & ||\frac{1}{n} \sum_{t=1}^n \cv_t -\frac{1}{T} \sum_{t=1}^T \cv_t|| + d(\frac{1}{T} \sum_{t=1}^T \cv_t , S) \\
& \le & \gamma+ (\delta- \gamma) = \delta,
\end{eqnarray*}
where we used the concentration bouds from Lemma \ref{lem:conc} to bound $||\frac{1}{n} \sum_{t=1}^n \cv_t -\frac{1}{T} \sum_{t=1}^T \cv_t||$ by $\gamma$. 
Therefore, $\{\cv_t\}_{t=1}^n$ is a primal feasible solution of $\OPTest^\delta(n)$ with objective value $f(\frac{1}{n}\sum_{t=1}^n \cv_t) \ge f(\frac{1}{T}\sum_{t=1}^T \cv_t) - L||\frac{1}{n} \sum_{t=1}^n \cv_t -\frac{1}{T} \sum_{t=1}^T \cv_t|| \ge f(\frac{1}{T}\sum_{t=1}^T \cv_t) - L\gamma = \OPT^{\delta-\gamma} - L\gamma$. Therefore, $\OPTest^\delta \ge f(\frac{1}{n}\sum_{t=1}^n \cv_t)\ge \OPT^{\delta-\gamma} - L\gamma$.

\end{proof}
\begin{lemma}
\label{lem:2}
For all $\rho>0 $ and $n\in [T]$, let  $\gamma=\gammaDefRho{n}$. For all $\delta \ge \gamma$, with probability $1-O(\rho)$,
$$\OPT^{\delta} + L\gamma \ge \OPTest^{\delta-\gamma}(n).$$
\end{lemma}
\begin{proof}
Define $S^{\delta}$ as the set $\{\cv: d(\cv, S) \le \delta\}$. Then, using the derivation in Equation \eqref{eq:OPTexpr}, we have that
\begin{eqnarray*}
\OPT^\delta & = & \min_{\lambda \ge 0, ||\phiV||_*\le L, ||\thetaV||_*\le 1} \left\{\right. f^*(\phiV) + \lambda h_{S^\delta}(\thetaV) \\
& & \ \ \ + \frac{1}{T} \sum_{t=1}^T h_{A_t}(-\phiV-\lambda \thetaV) \left.\right\}.
\end{eqnarray*}
Let $\lambda^*$, $\thetaV^*$, $\phiV^*$ be the optimal dual solutions in above. Then,
\begin{eqnarray*}
\OPTest_\dist^{\delta-\gamma}(n) & = & \min_{\lambda \ge 0, ||\phiV||_*\le L, ||\thetaV||_*\le 1} \left\{\right. f^*(\phiV) + \lambda h_{S^{\delta-\gamma}}(\thetaV) \\
& & \ \ + \frac{1}{n} \sum_{t=1}^n h_{A_t}(-\phiV-\lambda \thetaV) \left.\right\}\\
& \le &  f^*(\phiV^*) + \lambda^* h_{S^{\delta-\gamma}}(\thetaV^*) \\
& & \ \ + \frac{1}{n} \sum_{t=1}^n h_{A_t}(-\phiV^*-\lambda^* \thetaV^*) 
\end{eqnarray*}
Now, using concentration bounds from Lemma \ref{lem:conc} for the sum of $h_{A_t}$'s, we obtain,
\begin{eqnarray*}
\OPTest^{\delta-\gamma}(n) & \le &  f^*(\phiV^*) + \lambda^* h_{S^{\delta-\gamma}}(\thetaV^*) \\
& & \ \ + \frac{1}{T} \sum_{t=1}^T h_{A_t}(-\phiV^*-\lambda^* \thetaV^*) \\
& & \ \ + \gamma (\lambda^*||\thetaV^*||_*+||\phiV^*||_*).
\end{eqnarray*}
Now, observe that for any $\thetaV$, $h_{S^{\delta}}(\thetaV) \ge h_{S^{\delta-\gamma}}(\thetaV) + \gamma ||\thetaV||_*$. To see this, let $\cv$ be the maximizer in the definition of $h_{S^{\delta-\gamma}}$, i.e.,
$\cv=\arg \max_{{\bf u}\in S^{\delta-\gamma}} {\bf u} \cdot \thetaV$. Then consider $\cv'=\cv+ \gamma\frac{\thetaV}{||\thetaV||}$. We have that $||\cv'-\cv|| = \gamma$, so that $\cv\in S^{\delta-\gamma}$ implies that $\cv\in S^{\delta}$. Therefore $h_{S^\delta}(\thetaV) \ge \cv'\cdot \thetaV = \cv\cdot \thetaV + \gamma ||\thetaV||_* = h_{S^{\delta-\gamma}}(\thetaV)+ \gamma ||\thetaV||_*$. Substituting, we get,
\begin{eqnarray*}
\OPTest^{\delta-\gamma}(n) & \le &  f^*(\phiV^*) + \lambda^* h_{S^{\delta}}(\thetaV^*) - \gamma \lambda^* ||\thetaV^*||_*\\
& & \ + \frac{1}{T} \sum_{t=1}^T h_{A_t}(-\phiV^*-\lambda^* \thetaV^*) \\
& & \ + \gamma(\lambda^*||\thetaV^*||_*+||\phiV^*||_*)\\
& = & \OPT^\delta + \gamma ||\phiV^*||_*\\
& \le & \OPT^\delta + \gamma L
\end{eqnarray*}
\end{proof}

\begin{lemma}
\label{lem:3}
For all $\delta \ge \gamma$, with probability $1-O(\rho)$,
$$\OPTest^\delta(n) \le \OPT + 2\delta (L+Z^*)$$
\end{lemma}
\begin{proof}
Using the derivations in Equation \eqref{eq:OPTexpr},
\begin{eqnarray*}
\OPTest_\dist^{\delta}(n) & = & \min_{\lambda \ge 0, ||\phiV||_*\le L, ||\thetaV||_*\le 1} \left\{ \right. f^*(\phiV) + \lambda h_S(\thetaV) \\
& & \ + \frac{1}{n} \sum_{t=1}^n h_{A_t}(-\phiV-\lambda \thetaV) + \delta \lambda \left. \right\},
\end{eqnarray*}
Let $\lambda^*, \phiV^*, \thetaV^*$ denote the optimal dual solution for $\OPT$, then,
\begin{eqnarray*}
\OPTest_\dist^{\delta}(n) & \le & f^*(\phiV^*) + \lambda h_S(\thetaV^*) \\
& & \ + \frac{1}{n} \sum_{t=1}^n h_{A_t}(-\phiV^*-\lambda^* \thetaV^*) + \delta  \lambda^*\\
\end{eqnarray*}
Now, using concentration bounds from Lemma \ref{lem:conc} for the sum of $h_{A_t}$'s, we obtain,
\begin{eqnarray*}
\OPTest_\dist^{\delta}(n) & \le & f^*(\phiV^*) + \lambda h_S(\thetaV^*) \\
& & \ + \frac{1}{T} \sum_{t=1}^T h_{A_t}(-\phiV^*-\lambda^* \thetaV^*) \\
& & \ + \gamma (\lambda^*||\thetaV^*||_*+||\phiV^*||_*)+ \delta  \lambda^*\\
& = & \OPT + \gamma (\lambda^*||\thetaV^*||_*+||\phiV^*||_*)+ \delta  \lambda^*\\
& \le & \OPT +  (L+\lambda^*) \gamma +\delta \lambda^*\\
& \le & \OPT +  2(L+\lambda^*) \delta\\
& = & \OPT + 2\delta (L+Z^*)
\end{eqnarray*}

\end{proof}

\section{Proof of \prettyref{lem:optestimation}} \label{app:optestimation}

Given an instance of the online packing problem, recall that  $(\ropt, \vopt)$ denotes the optimal offline solution.  Then $\optsum = \sum_{t=1}^T \ropt$, 
and $\sum_{t=1}^{T} \vopt \leq B\ones$. 
Given $\rho > 0$, let $\eta = \sqrt{3\log(\tfrac {d+2} \rho)}$. 
Let the given random subset  of $\delta$ fraction of requests be $\Gamma$. 
{\em Define  
$\opthat$ to be $1/\delta$ times the optimum value of the following scaled optimization problem}: pick $(\rplay, \vplay)$ for each $t \in \Gamma$, to maximize the total reward $\sum_{t\in \Gamma} \rplay$ such that $\sum_{t\in \Gamma} \vplay \leq (\delta B + \eta \sqrt{\delta B})\ones  $. 

The bounds we need on $\opthat$ follow from considering the  optimal primal and dual solutions to the given packing problem restricted to the sample and using  Corollary \ref{cor:multiplicativeChernoff} to bound their values on the sample. 
Applying Corollary \ref{cor:multiplicativeChernoff} to the set of $\ropt$ for all $t \in [T]$ we get that with probability at least $1-\rho/(d+2)$, 
\begin{eqnarray*}
 \sum_{t \in \Gamma} \ropt & \geq & \delta \optsum - \sqrt{3\delta \optsum \log (\tfrac {d+2} \rho)} \\
& = & \delta \optsum - \eta \sqrt{\delta \optsum } . 
\end{eqnarray*}
Similarly, applying Corollary \ref{cor:multiplicativeChernoff} to each co-ordinate of the set of $\vopt$s, and taking a union bound, we get that 
with probability at least $1- \rho d/(d+2)$, 
\begin{eqnarray*}
 \sum_{t \in \Gamma} \vopt & \leq  & (\delta B + \sqrt{3\delta B \log (\tfrac {d +2} \rho)} ) \ones \\
& =  & (\delta B + \eta \sqrt{\delta B } ) \ones. 
\end{eqnarray*}
Therefore with probability $1- \rho (d+1)/(d+2)$ both the inequalities above hold and $(\ropt,\vopt)_{t\in \Gamma}$ is a feasible solution to the scaled optimization problem used to define $\opthat$. Hence 
\[ \delta \opthat \geq \sum_{t \in \Gamma} \ropt \geq \delta \optsum - \eta \sqrt{\delta \optsum } \] 
and the first bound on $\opthat$ follows from dividing the above inequality throughout by $\delta$. 
For the second bound, we need to consider the dual of the packing problem. 
The packing problem has the following natural LP relaxation. (The dual LP follows.)

\newcommand{\vgen}{\mathbf{v}}
	\begin{align*}
\max &\sum_{t=1}^T \sum_{ \vgen \in A_t} r(\vgen )x_{t,\vgen} \\
\text{s.t.} & \forall~ t, \sum_{\vgen \in A_t}  x_{t,\vgen} \leq 1 \\ 
&  \sum_{t=1}^T \sum_{\vgen \in A_t} \vgen x_{t,\vgen} \leq B\ones.  \\ 
\end{align*}
	\begin{align*}
	\min &\sum_{t=1}^T \beta_t + B \thetaV\cdot \ones \\
\text{s.t.} & \forall~ t, \forall~\vgen\in A_t, \beta_t \geq r(\vgen)  - \vgen\cdot {\thetaV}, \\ 
& \forall~t, \beta_t\geq 0, \thetaV \geq 0.
\end{align*}
First of all, we ignore the \revision{integrality gap and assume that the value of the optimal dual (and primal) solution is equal to the optimal value $\optsum$ for the offline packing problem.} 
Let $(\beta_t^*)_{t=1}^T , (\theta_j^*)_{j=1}^d$ be the optimal dual solution for 
the given instance, and $\optsum = \sum_t \beta_t^* + \sum_j B\theta_j^* $. It can be shown that $\beta_t^* \in [0,1]$ for all $t$:   all the constraints involving $\beta_t$ are of the form $\beta_t \geq (\cdot)$ so at least one of these constraints is tight for the optimal solution. Also for each of these constraints, the RHS is at most $1$, and one of the constraints is $\beta_t \geq 0$. Further note that these constraints are local, i.e., they only depend on the request indexed by $t$. This means that 
$(\beta_t^*)_{t\in \Gamma}, (\theta_j^*)_{j=1}^d$ is a feasible solution to the dual of the scaled optimization problem. The objective value of this solution to this dual is 
\[ \sum_{t\in \Gamma} \beta_t^* + \sum_j (\delta B + \eta \sqrt{\delta B}) \theta_j^* \geq \delta \opthat.\]
Using Corollary \ref{cor:multiplicativeChernoff} on the set of $\beta_t^*$s, we get that with probability at least $1-\rho/(d+2)$, 
\begin{eqnarray*}
 \sum_{t \in \Gamma} \beta_t^* & \leq & \delta \sum_{t =1}^T \beta_t^*  + \sqrt{3\delta \optsum \log (\tfrac {d+2} \rho)} \\
& =  & \delta  \sum_{t =1}^T \beta_t^* + \eta \sqrt{\delta \optsum } . 
\end{eqnarray*}
Putting the two inequalities above together, 
\begin{eqnarray*}
 \frac{\delta \opthat} { 1 + \eta /\sqrt{\delta B}} & \leq & \sum_{t\in \Gamma} \beta_t^* + \delta \sum_j B \theta_j^*  \\
& \leq  & \delta  \left(\sum_{t =1}^T \beta_t^* + \sum_j B \theta_j^*\right)+ \eta \sqrt{\delta \optsum }  \\
& = & \delta \optsum + \eta \sqrt{\delta \optsum }. 
\end{eqnarray*}

The lemma follows by taking the union bound over the probabilities for the two inequalities as required. 
Finally, we ignored the \revision{integrality} gap, but it is easy to show that this gap is at most $1- \tfrac 1 B$, which can be absorbed in the $1 + \eta /\sqrt{\delta B}$ factor.

\end{document}